\pdfoutput=1
\documentclass[11pt,letterpaper]{article}
\usepackage{fullpage}
\usepackage{blindtext}
\usepackage{hyperref}
\hypersetup{
	colorlinks=true,
	linkcolor=blue!70!black,
	citecolor=blue!70!black,
	urlcolor=blue!70!black
}

\usepackage[english]{babel}
\usepackage[T1]{fontenc}
\usepackage{tablefootnote}
\usepackage[table]{xcolor}
\usepackage{tabularx}
\newcolumntype{Y}{>{\centering\arraybackslash}X}
\usepackage{amsmath}
\usepackage{todonotes}
\usepackage{amssymb}
\usepackage{amsthm}
\usepackage{booktabs}
\usepackage{accents}
\usepackage{algorithm}
\usepackage{bbm}
\usepackage{bbold}
\usepackage{bm}
\usepackage{graphicx}
\graphicspath{{./figs/}}

\usepackage{cleveref}
\usepackage{thm-restate}
\usepackage{algorithm}
\usepackage{algpseudocode}

\newtheorem{theorem}{Theorem}
\newtheorem{lemma}{Lemma}
\newtheorem{fact}{Fact}
\newtheorem{definition}{Definition}
\newtheorem{corollary}{Corollary}
\newtheorem{proposition}{Proposition}

\newtheorem{problem}{Problem}

\newtheorem{remark}{Remark}
\newtheorem{model}{Model}
\newtheorem{construction}{Construction}

\newcommand{\defeq}{:=}
\newcommand{\norm}[1]{\left\lVert#1\right\rVert}
\newcommand{\norms}[1]{\lVert#1\rVert}
\newcommand{\normop}[1]{\left\lVert#1\right\rVert_{\textup{op}}}
\newcommand{\norminf}[1]{\left\lVert#1\right\rVert_{\infty, \infty}}
\newcommand{\normsinf}[1]{\lVert#1\rVert_{\infty, \infty}}
\newcommand{\normf}[1]{\left\lVert#1\right\rVert_{\textup{F}}}

\newcommand{\normsop}[1]{\lVert#1\rVert_{\textup{op}}}
\newcommand{\normsf}[1]{\lVert#1\rVert_{\textup{F}}}

\newcommand{\inprod}[2]{\left\langle#1, #2\right\rangle}
\newcommand{\inprods}[2]{\langle#1, #2\rangle}
\newcommand{\eps}{\epsilon}
\newcommand{\lam}{\lambda}
\newcommand{\0}{\boldsymbol{0}}

\newcommand{\R}{\mathbb{R}}
\newcommand{\Prob}{\mathbb{P}}

\newcommand{\N}{\mathbb{N}}

\newcommand{\vt}{\tilde{\vv}}
\newcommand{\hS}{\widehat S}

\newcommand{\bas}[1]{\begin{align*}#1\end{align*}}
\newcommand{\ba}[1]{\begin{align}#1\end{align}}
\newcommand{\bbb}[1]{\left[#1\right]}

\newcommand{\diag}[1]{\textbf{\textup{diag}}\left(#1\right)}
\newcommand{\diags}[1]{\textbf{\textup{diag}}(#1)}
\newcommand{\half}{\frac{1}{2}}

\newcommand{\bk}{\color{black}}
\newcommand{\rd}{\color{red}}
\newcommand{\bl}{\color{blue}}
\newcommand{\1}{\boldsymbol{1}}
\newcommand{\E}{\mathbb{E}}

\newcommand{\Nor}{\mathcal{N}}

\newcommand{\Tr}{\textup{Tr}}

\newcommand{\simiid}{\stackrel{\mathrm{i.i.d.}}{\sim}}
\newcommand{\simunif}{\stackrel{\mathrm{unif.}}{\sim}}
\newcommand{\eqdist}{\stackrel{\mathrm{d}}{=}}
\newcommand{\event}{\calE}
\newcommand{\ball}{\mathbb{B}}
\newcommand{\ma}{\mathbf{A}}
\newcommand{\hmsig}{\widehat{\msig}}
\newcommand{\mm}{\mathbf{M}}
\newcommand{\mw}{\mathbf{W}}
\newcommand{\me}{\mathbf{E}}

\newcommand{\id}{\mathbf{I}}
\newcommand{\bb}[1]{\left(#1\right)}

\newcommand{\dd}{\textup{d}}

\usepackage{xcolor}
\definecolor{burntorange}{rgb}{0.8, 0.33, 0.0}

\newcommand{\syamantak}[1]{\textcolor{blue}{\textbf{syamantak:} #1}}
\newcommand{\purna}[1]{\textcolor{magenta}{\textbf{purna:} #1}}
\newcommand{\shourya}[1]{\textcolor{purple}{\textbf{shourya:} #1}}
\newcommand{\tO}{\widetilde{O}}
\newcommand{\nnz}{\textup{nnz}}

\newcommand{\Par}[1]{\left(#1\right)}
\newcommand{\Brack}[1]{\left[#1\right]}
\newcommand{\Brace}[1]{\left\{#1\right\}}
\newcommand{\Abs}[1]{\left|#1\right|}
\newcommand{\abs}[1]{\left|#1\right|}
\newcommand{\Pars}[1]{(#1)}
\newcommand{\Bracks}[1]{[#1]}
\newcommand{\Braces}[1]{\{#1\}}
\newcommand{\Abss}[1]{|#1|}
\newcommand{\Cov}{\textup{Cov}}

\newcommand{\alg}{\mathcal{A}}

\newcommand{\mb}{\mathbf{B}}
\newcommand{\mg}{\mathbf{G}}

\newcommand{\my}{\mathbf{Y}}
\newcommand{\mv}{\mathbf{V}}
\newcommand{\hmv}{\widehat{\mv}}

\newcommand{\msig}{\boldsymbol{\Sigma}}

\newcommand{\mPsi}{\boldsymbol{\Psi}}

\newcommand{\tmsig}{\widetilde{\msig}}

\newcommand{\mx}{\mathbf{X}}

\newcommand{\mmu}{\mathbf{U}}

\newcommand{\mzero}{\mathbf{0}}

\newcommand{\Sym}{\mathbb{S}}
\newcommand{\supp}{\textup{supp}}

\newcommand{\PSD}{\Sym_{\succeq \mzero}}
\newcommand{\PD}{\Sym_{\succ \mzero}}
\DeclareMathOperator{\linspan}{span}
\newcommand{\mn}{\mathbf{N}}

\newcommand{\hgam}{\widehat{\gamma}}

\newcommand\mycommfont[1]{{\footnotesize\textcolor{gray}{#1}}}

\algrenewcommand\algorithmiccomment[1]{%
  \hfill\mycommfont{$\triangleright$~#1}%
}

\newcommand{\md}{\mathbf{D}}

\newcommand{\inner}{\inprod}

\newcommand{\poly}{\textup{poly}}

\newcommand{\calP}{\mathcal{P}}
\newcommand{\polylog}{\textup{polylog}}

\newcommand{\sign}{\mathsf{sign}}
\newcommand{\calE}{\mathcal{E}}

\newcommand{\Prb}{\mathbb{P}}
\newcommand{\TV}{\mathrm{TV}}

\newcommand{\ham}{d_{\textup{ham}}}
\newcommand{\Lap}{\mathsf{Lap}}
\newcommand{\BL}{\mathsf{BoundedLaplace}}
\newcommand{\BLap}{\mathsf{BoundedLap}}
\newcommand{\InvWishart}{\mathsf{InvWishart}}

\newcommand{\gap}{\mathsf{gap}}
\newcommand{\ExpMechPCA}{\mathsf{ExpMechPCA}}
\newcommand{\kRCS}{k\text{-RCS}}

\DeclareMathOperator*{\argmin}{arg\,min}

\newcommand{\mpp}{\mathbf{P}}
\newcommand{\hmpp}{\widehat{\mpp}}
\newcommand{\tmpp}{\widetilde{\mpp}}

\newcommand{\muu}{\mathbf{U}}
\newcommand{\mz}{\mathbf{Z}}
\newcommand{\hlam}{\hat{\lambda}}

\newcommand{\vtheta}{\bm{\theta}}
\newcommand{\va}{\mathbf{a}}
\newcommand{\vb}{\mathbf{b}}
\newcommand{\vc}{\mathbf{c}}

\newcommand{\ve}{\mathbf{e}}

\newcommand{\vr}{\mathbf{r}}
\newcommand{\vs}{\mathbf{s}}
\newcommand{\vu}{\mathbf{u}}
\newcommand{\vv}{\mathbf{v}}
\newcommand{\vw}{\mathbf{w}}
\newcommand{\vx}{\mathbf{x}}
\newcommand{\vy}{\mathbf{y}}
\newcommand{\vz}{\mathbf{z}}

\newcommand{\hvv}{\widehat{\vv}}

\newcommand{\calA}{\mathcal{A}}

\newcommand{\calC}{\mathcal{C}}
\newcommand{\calD}{\mathcal{D}}

\newcommand{\calG}{\mathcal{G}}

\newcommand{\calM}{\mathcal{M}}
\newcommand{\calN}{\mathcal{N}}

\newcommand{\calR}{\mathcal{R}}
\newcommand{\calS}{\mathcal{S}}
\newcommand{\calT}{\mathcal{T}}
\newcommand{\calU}{\mathcal{U}}
\newcommand{\calV}{\mathcal{V}}

\newcommand{\calX}{\mathcal{X}}
\newcommand{\calY}{\mathcal{Y}}

\newcommand{\Top}{\textup{top}}

\newcommand{\thresh}{\calT}
\newcommand{\independent}{\perp\!\!\!\!\perp}
\newcommand{\op}{\textup{op}}

\newcommand{\Unif}{\mathsf{Unif}}

\newcommand{\kRCSPCA}{\mathsf{PrivPCA}}
\newcommand{\kRCSCOV}{\mathsf{PrivCov}}
\newcommand{\kRCSOPNORM}{\mathsf{PrivNorm}}
\newcommand{\dsign}{d_{\mathrm{sign}}}

\newcommand{\ind}{\mathbbm{1}}

\newcommand{\hvs}{\widehat{\vs}}
\newcommand{\dirE}{\overrightarrow{E}}
\newcommand{\bmsig}{\bar{\msig}}
\newcommand{\FC}{\mathsf{FriendlyCore}}

\newcommand{\lamsamp}{\lambda_{\rm samp}}
\newcommand{\lampriv}{\lambda_{\rm priv}}
\newcommand{\epsrow}{\eps_{\rm row}}
\newcommand{\deltarow}{\delta_{\rm row}}

\newcommand\blfootnote[1]{%
  \begingroup
  \renewcommand\thefootnote{}\footnote{#1}%
  \addtocounter{footnote}{-1}%
  \endgroup
}

\title{On the Curse of Dimensionality in \\ Private Sparse Covariance Estimation and PCA}
\date{}
\author{
Syamantak Kumar\thanks{University of Texas at Austin, \texttt{syamantak@utexas.edu}} \and
Shourya Pandey\thanks{University of Texas at Austin, \texttt{shouryap@utexas.edu}} \and
Purnamrita Sarkar\thanks{University of Texas at Austin, \texttt{purna.sarkar@austin.utexas.edu}} \and
Kevin Tian\thanks{University of Texas at Austin, \texttt{kjtian@cs.utexas.edu}}
}
\begin{document}

\maketitle

\begin{abstract}
We study high-dimensional  \emph{differentially private} (DP) covariance estimation in the operator norm, and principal component analysis (PCA), under $k$-\emph{row-column sparsity} ($k$-RCS) of the covariance matrix. In the non-private setting, it is known that $\poly(k, \log d)$ samples suffice to solve both of these problems. However, the only comparable result known under DP \cite{wang2021differentially} requires $\Omega(d)$ samples under standard parameterizations of the problem.
We investigate when this curse of dimensionality is inherent for sparse covariance estimation tasks under DP.

On the upper bound front, we show that a $\poly(k, \log d)$ sample complexity for PCA is possible under DP, if we also posit sparsity of the leading eigenvector. We complement this result with $\poly(d)$ lower bounds under DP for both sparse covariance estimation and PCA, establishing an \emph{exponential gap} between the private and non-private variants of these problems when $k = \polylog(d)$. To our knowledge, no such separation has previously been demonstrated for any sparse estimation problems in private high-dimensional statistics. Our techniques are flexible enough that they imply stronger lower bounds even for the well-studied problem of standard DP PCA, without sparsity assumptions.\blfootnote{Accepted for presentation at the
Conference on Learning Theory (COLT) 2026}
\end{abstract}
\thispagestyle{empty}

\newpage

\tableofcontents
\thispagestyle{empty}

\newpage

\section{Introduction}
\setcounter{page}{1}

We study covariance estimation and principal component analysis (PCA) in regimes
where the ambient dimension $d$ is comparable to or much larger than the sample size $n$.\footnote{Throughout, PCA refers to $1$-PCA, i.e., the problem of recovering the
leading (rank-$1$) principal component.}
In this setting, classical estimators such as the sample covariance matrix and standard PCA are no
longer reliable and can be provably inconsistent~\cite{johnstoneLu2009,baik2005bbp,paul2007asymptotics}.
Thus in high dimensions, meaningful inference requires additional structural assumptions on the covariance matrix, such as row sparsity, sparsity of the principal eigenspace, or a combination thereof~\cite{Bickel2009CovarianceRB,cai2010optimal}.
These assumptions have motivated a rich literature on sparse high-dimensional covariance estimation
and sparse PCA, yielding procedures with strong statistical estimation guarantees in settings where classical methods
break down~\cite{johnstoneLu2009,vu2013minimax,BerthetR13,berthet2016complexity,amini2008high,10.1214/13-AOS1097,deshpande2016sparse, kumar2024oja,qiu2023gradientbased}.

At the same time, many applications involving sensitive data require rigorous differential privacy (DP) guarantees. Compared to the vast literature on differentially private PCA~\cite{chaudhuri2013near, dwork2014analyze, liu2022dp}, there have been surprisingly few works that have investigated private \textit{sparse} PCA.  Designing differentially private algorithms for high-dimensional covariance estimation and sparse PCA is
particularly challenging, as these tasks inherently induce large sensitivities: when few samples are taken, each contributes more heavily to empirical statistics.

Several recent works~\cite{ge2018minimax,wang2021differentially,li2023differentially} have studied differentially private sparse PCA and covariance estimation, under a non-standard parameterization that each data point is uniformly bounded in $\ell_2$ norm. This condition enables worst-case sensitivity control and leads naturally to mechanisms that add noise to the empirical covariance matrix, followed by truncation-based post-processing. However, such boundedness assumptions can be overly restrictive in high dimensions. For instance, if $\vx_i \sim \mathcal{N}(\0,\id_d)$, then $\|\vx_i\|_2$ concentrates on the order of $\sqrt{d}$, so enforcing a unit-norm bound effectively obscures the dependence on $d$. Under a natural scale-invariant guarantee that factors in the sub-Gaussian parameter (see Model~\ref{model:general_krcs_cov}), the best prior result by \cite{wang2021differentially} has a sample complexity of $\Omega(dk^2)$,\footnote{Under the scaling discussed near Eq.\ (2) of \cite{wang2021differentially}, their distribution is $\sigma^2 = O(\frac 1 d)$-sub-Gaussian, so achieving error as in Model~\ref{model:general_krcs_cov} inflates their sample complexity as stated in their Theorem 2 by a $\frac 1 {\sigma^2} = \Omega(d)$ factor.} substantially larger than the best $\poly(k, \log(d))$ sample complexities for solving the same task without privacy.

This state of affairs suggests that there may be a \emph{curse of dimensionality} that is specific to requiring DP in sparse covariance estimation tasks. Our motivation is precisely to understand whether this gap is inherent under DP, and whether there are additional natural structural assumptions that, when imposed, alleviate the curse of dimensionality. Our main questions are thus as follows.
\begin{gather*}\textit{Can we prove $\Omega(\poly(d))$ lower bounds for sparse covariance estimation tasks under DP?} \\
\textit{Conversely, can we achieve $\poly(k, \log(d))$ sample complexities under additional structure?}
\end{gather*}

\subsection{Our results}

We develop a suite of new DP algorithms and lower bounds for \emph{sparse covariance estimation}
and \emph{sparse PCA}, two canonical problems in high-dimensional statistics, under the following models. We refer the reader to Section~\ref{sec:prelim} for preliminaries on differential privacy and matrix concentration.

First, our general Model~\ref{model:general_krcs_cov} requires a $\kRCS$ (i.e., with $k$-sparse rows and columns, see Definition~\ref{def:rcs}) covariance structure without any particular assumptions on the top eigenvector.

\begin{model}[$k$-sparse covariance model]\label{model:general_krcs_cov} Fix $(k, d, n) \in \N^3$ with $k \in [d]$, $\gamma \in [0, \half)$, and $\sigma > 0$. In the \emph{$k$-sparse covariance model}, there is an unknown $\kRCS$ covariance matrix $\msig \in \PSD^{d \times d}$,  with a leading eigenvector $\vv_{1} \in \R^{d}$, and eigenvalues satisfying $\lam_2(\msig) \leq (1-\gamma)\lam_1(\msig)$.
We obtain samples $\{\vx_i\}_{i \in [n]} \simiid \calD$, a $\sigma$-sub-Gaussian distribution with covariance $\msig$.
\end{model}

Note that under Model~\ref{model:general_krcs_cov}, the top eigenvector $\vv_1$ is well-defined iff the gap parameter satisfies $\gamma > 0$.
Our next model, Model~\ref{model:general_krcs_pca}, additionally enforces sparsity of the top eigenvector $\vv_1$.
\begin{model}[$k$-sparse PCA model]\label{model:general_krcs_pca}
Instantiate Model~\ref{model:general_krcs_cov} where $\gamma > 0$, and let $\nnz(\vv_1) \le k$.
\end{model}

Model~\ref{model:general_krcs_pca} is a natural form of additional structure to impose under a $\kRCS$ covariance assumption: indeed, many existing works on sparse PCA phrase the problem with this additional requirement.
Given samples from Model~\ref{model:general_krcs_cov} or Model~\ref{model:general_krcs_pca}, we study two different estimation problems regarding $\msig$.

\begin{problem}[Sparse covariance estimation]\label{prob:cov_est} Let $(\eps, \delta, \alpha, \beta) \in (0,1)^{4}$. Given $\sigma$-sub-Gaussian samples $\{\vx_i\}_{i \in [n]}$, the goal is to return an $(\eps, \delta)$-DP matrix $\hmsig \in \Sym^{d \times d}$ satisfying, with probability at least $1-\beta$,
\bas{
    \normop{\hmsig - \msig} \leq \alpha \sigma^{2} .
}
\end{problem}

We note that the normalization by $\sigma^2$ in Problem~\ref{prob:cov_est} is to maintain scale-invariance of our bounds.

\begin{problem}[Sparse PCA]\label{prob:pca} Let $(\eps, \delta, \Delta, \beta) \in (0,1)^{4}$. Given samples $\{\vx_i\}_{i \in [n]}$, the goal is to return an $(\eps, \delta)$-DP unit vector $\hvv \in \R^{d}$ satisfying, with probability at least $1-\beta$,
\bas{
    \sin^2\angle(\hvv, \vv_1) \leq \Delta .
}
\end{problem}

\paragraph{Private sparse covariance estimation.}

In Section~\ref{ssec:priv_sparse_cov_est}, we study Problem~\ref{prob:cov_est}, where we give new upper and lower bounds that agree in their dependence on $d$. We summarize our results in Table~\ref{tab:dp-cov-bounds}. Note that the assumption in Model~\ref{model:general_krcs_pca} does not affect the problem statement, and all of our bounds apply to the more general Model~\ref{model:general_krcs_cov}, so we do not differentiate the model for these results.
\begin{table}[t]
  \centering
  \caption{DP Sparse Covariance Estimation (Problem~\ref{prob:cov_est}). Logarithmic factors omitted, bounds stated for $\eps = \alpha = \beta = \Theta(1)$. All results hold under approximate $(\eps, \delta)$-DP.}
  \label{tab:dp-cov-bounds}

  {\small 
  \renewcommand{\arraystretch}{1.05} 
  \setlength{\tabcolsep}{6pt} 

  \begin{tabular}{@{}>{\bfseries}lccc@{}}
    \toprule
    \rowcolor{gray!12}
    & \textbf{Non-Private} & \textbf{Our Results} & \textbf{Prior Results} \\
    \midrule
    Upper Bound & $k^2$ \cite{Bickel2009CovarianceRB} &
    $k^2 + \sqrt d \cdot k^{1.5}$ (Thm.~\ref{thm:priv_sparse_cov_est_upper_bound})
    & $d\cdot k^2$ \cite{wang2021differentially} \\
    Lower Bound & $k^2$ \cite{cai2012sparsecov} &
    $k^2 + \sqrt d \cdot k$ (Thm.~\ref{thm:priv_sparse_cov_est_lower_bound})
    & None \\
    \bottomrule
  \end{tabular}
  } 
\end{table}

Our upper bound (Theorem~\ref{thm:priv_sparse_cov_est_upper_bound}) uses a simple thresholding algorithm, a standard strategy for Problem~\ref{prob:cov_est} in the non-private setting. Intuitively, the $\sqrt d$ factor appears from advanced composition, because we need to privately perform a top-$k$ selection step on each of $d$ rows. Interestingly, our lower bound (Theorem~\ref{thm:priv_sparse_cov_est_lower_bound}) shows that this $\sqrt d$ dependence is tight, even under approximate DP. Our bound is proven by adapting the fingerprinting techniques of \cite{narayanan2024better} for DP covariance estimation, and extending them to sparse models via a graph-based construction (Construction~\ref{const:fingerprinting_prior}). Our results again establish that for small $k$, there is an exponential separation between the private and non-private variants of Problem~\ref{prob:cov_est}, highlighting a curse of dimensionality specific to DP.

\paragraph{Private sparse PCA.} In Sections~\ref{sssec:priv_sparse_pca_upper_bound} and \ref{sssec:priv_sparse_pca_lower_bound}, we respectively give new upper and lower bounds for Problem~\ref{prob:pca}. Our upper bound in Theorem~\ref{thm:priv_sparse_pca_upper_bound_gamma} applies when the eigenvector sparsity assumption in Model~\ref{model:general_krcs_pca} holds, whereas our lower bounds consider the sample complexity of Problem~\ref{prob:pca} under the weaker Model~\ref{model:general_krcs_cov}. Our results are summarized in Table~\ref{tab:dp-pca-bounds}.

Our main upper bound (Theorem~\ref{thm:priv_sparse_pca_upper_bound_gamma}) demonstrates that under Model~\ref{model:general_krcs_pca}, $\poly(k, \log d)$ samples suffice to solve Problem~\ref{prob:pca} subject to approximate DP (omitting dependences on other parameters). Our algorithm uses the $\mathsf{FriendlyCore}$ primitive of~\cite{FriendlyCore-tsfadia22a} to construct a covariance estimate with $d$-independent sensitivity in the Frobenius norm. This allows us to add bounded noise entrywise via the Gaussian mechanism to our stable estimate, which negligibly affects the truncation typical in sparse covariance algorithms. By further leveraging the sparsity assumption in Model~\ref{model:general_krcs_pca}, we estimate the top eigenvector via a private support estimation step, concluding our result.

\begin{table}[t]
  \centering
  \caption{DP Sparse PCA (Problem~\ref{prob:pca}), logarithmic factors omitted, bounds stated for $\eps = \Delta = \beta = \gamma = \Theta(1)$. Upper bounds hold under $(\eps,\delta)$ DP; lower bounds hold under $(\eps, 0)$ DP.}
  \label{tab:dp-pca-bounds}

  {\small 
  \renewcommand{\arraystretch}{1.05} 
  \setlength{\tabcolsep}{6pt} 

  \begin{tabular}{@{}>{\bfseries}lccc@{}}
    \toprule
    \rowcolor{gray!12}
    & \textbf{Non-Private} & \textbf{Our Results} & \textbf{Prior Results} \\
    \midrule
    Upper Bound (Model~\ref{model:general_krcs_pca}) &
    $k^2$ \cite{Bickel2009CovarianceRB}\tablefootnote{This result holds even under Model~\ref{model:general_krcs_cov}.} &
    $k^4$ (Thm.~\ref{thm:priv_sparse_pca_upper_bound_gamma})
    & $d\cdot k^2$ \cite{wang2021differentially} \\
    Lower Bound (Model~\ref{model:general_krcs_cov}) &
    $k$ \cite{vu2013minimax} &
    $d$ (Thm.~\ref{thm:priv_sparse_pca_pure_dp_lower_bound})
    & None \\
    \bottomrule
  \end{tabular}
  } 
\end{table}

In light of Theorem~\ref{thm:priv_sparse_pca_upper_bound_gamma}, a natural question is to ask whether a similar $\poly(k, \log d)$ sample complexity is attainable without the stronger assumption in  Model~\ref{model:general_krcs_pca}. Indeed, in the non-private setting, no sparsity assumption on $\vv_1$ is needed at all, beyond arising from a $\kRCS$ covariance, for $\approx k^2$ samples to solve Problem~\ref{prob:pca} (say, when $\gamma = \Theta(1)$). Our next result, Theorem~\ref{thm:priv_sparse_pca_pure_dp_lower_bound}, dashes these hopes at least with regards to pure DP ($\delta = 0$), by showing that under this restriction, $\gtrsim d$ samples are necessary. Our lower bound is based on a packing argument (Lemma 6.2, \cite{kamath2020private}) and a coding-theoretic construction of $\exp(\Omega(d))$ covariance matrices that are simultaneously $\kRCS$ and have dense leading eigenvectors. For small  $k$ (e.g., $k = \polylog(d)$), our lower bound shows an \emph{exponential gap} between the sample complexity of the non-private and (pure) DP versions of this problem.

Previously, exponential separations have been established for several other problems in private statistics, e.g., $\ell_\infty$ mean estimation and hypothesis selection \cite{bun2014fingerprinting, steinke2017tight}.\footnote{We also mention conceptually-related results by \cite{KasiviswanathanLNRS11, CheuU21, NissimY22}, which showed similar exponential gaps under variants of DP, particularly, the local or shuffle DP models.} However, Theorem~\ref{thm:priv_sparse_pca_pure_dp_lower_bound} is the first such exponential separation for a natural \emph{sparse estimation task in high dimensions}. Our result thus has an important qualitative message: indeed, arguably the central question in sparse estimation is whether a $\poly(d)$ sample complexity is avoidable (when, say, $k = O(\polylog(d))$. Our Theorems~\ref{thm:priv_sparse_pca_upper_bound_gamma} and~\ref{thm:priv_sparse_pca_pure_dp_lower_bound} demonstrate that the modeling assumptions can provably shift the sample complexity landscape for DP problems in sparse covariance estimation.


One notable caveat is that the privacy guarantees in Theorems~\ref{thm:priv_sparse_pca_upper_bound_gamma} and~\ref{thm:priv_sparse_pca_pure_dp_lower_bound} do not precisely match, in the sense that our upper bound holds under approximate DP, whereas our lower bound is for pure DP. This potentially creates an opportunity for approximate DP algorithms to solve the general variant of Problem~\ref{prob:pca} (i.e., under Model~\ref{model:general_krcs_cov}) using only $\poly(k, \log d)$ samples.

Towards investigating this possibility, in Theorem~\ref{thm:priv_sparse_pca_approx_dp_lower_bound} we give another lower bound, this time under approximate DP. Our bound applies to PCA in a family of $\kRCS$ covariance matrices $\msig$, such that the resulting distribution yields samples with norms $\approx \sqrt d$ times $\normsop{\msig}$. We show that $\gtrsim d/ \eps$ samples are required to solve the problem even under approximate DP, by adapting the private Assouad's method of \cite{pmlr-v132-acharya21a}.
Unfortunately, the resulting sample distribution is $O(\sqrt{d})$-sub-Gaussian in the sense of Model~\ref{model:general_krcs_cov}, as enforcing sparsity induces certain spiky directions. Thus, this parameterization does not match our upper bound in Theorem~\ref{thm:priv_sparse_pca_upper_bound_gamma}. Nonetheless, our results broaden our understanding on the achievability of $\poly(k, \log d)$ sample complexities under different models. We leave proving a $\poly(d)$ lower bound, or excitingly, a $\poly(k, \log d)$ upper bound, under the approximate DP variant of Model~\ref{model:general_krcs_cov} as an interesting open problem.

En route to proving our results for Problem~\ref{prob:pca}, we also give a stronger lower bound for DP PCA (in the standard, non-sparse, setting) than prior works. Specifically, Theorem~\ref{thm:dense_PCA_approx_DP_lower_bound} demonstrates that $\gtrsim d$ samples are needed under approximate DP, for a much broader parameter range than previously known: e.g., Section 4.2, \cite{cai2024optimaldifferentiallyprivatepca} and Theorem 5.4, \cite{liu2022dp} only result in comparable bounds in the restrictive setting $\delta = \exp(-\Omega(d))$. We believe this result is of independent interest, as it improves our understanding of the tractability of an extremely well-studied problem.

\subsection{Related work}

\paragraph{High-dimensional covariance estimation and sparse PCA (non-private).}
Classical PCA can be statistically inconsistent in modern high-dimensional regimes, motivating structural assumptions such as sparsity of the covariance or of leading eigenvectors. For sparse covariance estimation, a large line of work studies thresholding and related regularization procedures that achieve dimension-free (or near dimension-free) rates under suitable sparsity/regularity conditions, including early thresholding estimators and their refinements
\cite{Bickel2009CovarianceRB,rothman2009generalized,cai2011adaptive, deshpande2016sparse}.
In parallel, sparse PCA has been extensively studied via optimization-based formulations and algorithmic relaxations, including $\ell_1$-penalized or regression-style approaches \cite{zou2006sparse},
semidefinite relaxations \cite{d2004direct,amini2008high},
and iterative schemes such as truncated power methods \cite{yuan2013truncated}.
A complementary thread establishes statistical limits and computational barriers in sparse PCA, clarifying when polynomial-time methods can (or cannot) attain minimax-optimal rates \cite{BerthetR13, berthet2016complexity}.

\paragraph{Differential privacy for subspace estimation when $d\leq n$.}
Differential privacy (DP) provides a rigorous framework for protecting individuals’ contributions in statistical analyses \cite{dr14}.
A core challenge in private high-dimensional problems is to control sensitivity while preserving spectral structure.
For private PCA and related spectral tasks, foundational results include tight privacy-utility analyses for PCA via Gaussian perturbations \cite{dwork2014analyze, cai2024optimaldifferentiallyprivatepca} and private iterative methods for dominant subspaces \cite{hardt2014noisy,liu2022dp, dungler2025an}.
More broadly, private low-rank approximation and private linear-algebraic primitives have been developed as building blocks for downstream tasks \cite{kapralov2013differentially}.
On the distribution-learning side, general-purpose private learners for high-dimensional structured families illuminate what is achievable when the ambient dimension is large, and boundedness assumptions are undesirable \cite{kamath2019privately}.

\paragraph{Private sparse/structured estimation and selection primitives.}
The intersection of privacy with sparsity brings additional algorithmic and information-theoretic constraints: even identifying the relevant support (or approximate support) can dominate the privacy budget.
For sparse covariance estimation under DP, prior work gives algorithms and rates under high-dimensional sparsity assumptions \cite{ge2018minimax, wang2021differentially, li2023differentially} with $O(1)$ norm bounded assumptions, which often does not hold for high-dimensional data.
Differentially private top-$k$ and sparse selection mechanisms, often used to locate large coordinates/entries before estimating magnitudes, have been studied extensively
\cite{durfee2019practical,qiao2021oneshot}.
These tools are particularly relevant for sparse PCA pipelines that must privately localize the support of a sparse leading eigenvector (or its projector) prior to accurate recovery.
Our results fit into this gap by giving end-to-end private procedures tailored to $k$-RCS structure (for covariance estimation) and sparse leading components (for PCA), combining structured truncation/thresholding with carefully calibrated noise so that the final guarantees scale primarily with the sparsity level rather than the ambient dimension.

\paragraph{Lower bounds: geometry, private minimax tools, and fingerprinting.}
On the impossibility side, DP lower bounds for high-dimensional estimation draw on geometric characterizations and packing arguments \cite{hardt2010geometry},
as well as DP analogues of classical minimax techniques (Assouad/Fano/Le Cam) \cite{pmlr-v132-acharya21a}.
Fingerprinting codes and their variants provide sharp lower bounds for answering many queries and for private statistical estimation, highlighting fundamental gaps between non-private and approximate-DP sample complexity \cite{hardt2010geometry, bun2014fingerprinting,steinke2015interactive}.
Recent work streamlines and strengthens the fingerprinting approach for modern estimation problems \cite{narayanan2024better}.
Our lower bounds build on and adapt these techniques to the covariance/PCA setting under the structural regimes considered here, yielding near-matching (up to logarithmic factors) separations that explain when sparsity-aware private procedures are necessary and when they are sufficient.
\section{Preliminaries}\label{sec:prelim}

\paragraph{General notation.}
We use $\widetilde{\Theta}(\cdot)$, $\tO(\cdot)$, and $\widetilde{\Omega}(\cdot)$ to hide polylogarithmic factors in the ambient problem parameters. We use $\poly(\cdot)$ and $\polylog(\cdot)$ for unspecified polynomial and polylogarithmic dependences. For $n \in \N$ we let $[n] \defeq \{i \in \N : i \le n\}$. For $a, b \in \N$, $a \mid b$ denotes that $a$ divides $b$. We denote vectors in lowercase boldface letters and matrices in capital boldface letters. We denote the $i^{\text{th}}$ canonical basis vector in $\R^d$ by $\ve_i$. For $p \in [1, \infty]$ we let $\norms{\vv}_p$ denote the $\ell_p$ norm of $\vv$. We use $\nnz$ to denote the number of nonzero entries in a vector or matrix, and $\supp$ to denote the support (i.e., indices of the nonzero entries). We use $\0_d$ to denote the all-zeroes vector and $\1_d$ to denote the all-ones vector in $\R^d$. For equal-length strings or vectors $\vx,\vy$, we write $\ham(\vx,\vy) \defeq \sum_i \ind(\vx_i \neq \vy_i)$ for their Hamming distance.

\paragraph{Probability notation.}
For an event $\calE$, we use $\ind(\calE)$ to denote the corresponding $0$-$1$ indicator random variable. For random variables $\vx, \vy$, we denote statistical independence by $\vx \independent \vy$.
 For two probability distributions $P,Q$ on the same measurable space, $\TV(P,Q) \defeq \sup_{\calA}\Abss{P(\calA)-Q(\calA)}$ denotes total variation distance. We write $X\sim P$ when $X$ has law $P$, and write $X\eqdist Y$ when $X$ and $Y$ have the same distribution. We write $\vx_1,\ldots,\vx_n \simiid P$ to denote independent samples with common law $P$, and $X \simunif S$ to denote that $X$ is uniform on a finite set $S$. 
 
 \paragraph{Notations on matrices.}
 For a symmetric matrix $\mm$, we order its eigenvalues as $\lam_1(\mm) \ge \lam_2(\mm) \ge \cdots$ and write $\vv_1(\mm)$ for a unit leading eigenvector when it is well-defined. In set-builder notation, we use a colon to separate the ambient set from the defining condition.
For $\mm \in \R^{m \times n}$ and $S \subseteq [m]$, $T \subseteq [n]$, we use $\mm_{S \times T}$ to denote the submatrix indexed by $S, T$. For matrices $\ma, \mb$ with the same number of rows, we let $\begin{pmatrix} \ma & \mb \end{pmatrix}$ denote their horizontal concatenation. We use $\ma_{i,:}, \ma_{:,i}$ to denote the $i^{\text{th}}$ row and column of matrix $\ma$. We let $\id_d$ be the $d \times d$ identity and $\mzero_{m \times n}$ be the all-zeroes $m \times n$ matrix. We let $\Sym^{d \times d}$ be the set of real symmetric $d \times d$ matrices, which we equip with the Loewner partial ordering $\preceq$ and the Frobenius inner product $\inprods{\mm}{\mn} = \Tr(\mm\mn)$. We let $\PSD^{d \times d}$ and $\PD^{d \times d}$ respectively denote the positive semidefinite and positive definite subsets of $\Sym^{d \times d}$. For matrix $\mm \in \R^{d \times d}$, we define $\normsop{\mm} := \sqrt{\lam_1(\mm\mm^{\top})}, \normsf{\mm} := \sqrt{\sum_{i,j \in [d]}\mm_{ij}^{2}}, \norms{\mm}_{\infty, \infty} := \sup_{i,j \in [d]}\Abss{\mm_{ij}}$. Denote by $\mathbb{B}_{\infty}\Pars{\mm, \tau}$ the set of matrices $\mm'$ satisfying $\norms{\mm - \mm'}_{\infty, \infty} \leq \tau$.

Our models consider estimation of covariance matrices satisfying the following definition.

\begin{definition}[$\kRCS$]\label{def:rcs}
Let $k \in [\min(m, n)]$. We say that a matrix $\mm \in \R^{m \times n}$ is $\kRCS$ ($k$-row-column sparse) if for all $i \in [m]$, $\nnz(\mm_{i,:}) \le k$ and for all $j \in [n]$, $\nnz(\mm_{:,j}) \le k$.
\end{definition}

We finally provide notation for procedures often used in the paper. For a vector $\vv \in \R^d$ and $k \in [d]$, we use $\Top_k(\vv)$ to denote the vector in $\R^d$ that zeroes out all but the top-$k$ entries of $\vv$ by magnitude (breaking ties arbitrarily). For a vector or matrix argument, and a threshold $\tau > 0$, we use $\thresh_\tau(\cdot)$ to be the vector or matrix that applies the following thresholding operation entrywise:
\begin{equation}\label{eq:thresh_def}\thresh_\tau(c) \defeq \begin{cases} c & |c| > \tau, \\ 0 & \text{otherwise.} \end{cases}.\end{equation}
For two unit vectors $\vu, \vv \in \R^{d}$, we define the $\sin^{2}$ error between them as
\begin{equation}\label{eq:sinesqerror}
\sin^2\angle(\vu, \vv) \defeq 1 - \frac{\inner{\vu}{\vv}^2}{\norm{\vu}_2^2\norm{\vv}_2^2}.
\end{equation}

\paragraph{Matrix concentration and linear algebra.} We begin by defining (multivariate) sub-Gaussianity.

\begin{definition}[$\sigma$-sub-Gaussianity]\label{def:subgaussianity}
A distribution $\calD$ supported on $\R^d$ is said to be $\sigma$-sub-Gaussian for proxy $\sigma > 0$ if, for any vector $\vu \in \R^d$,
\bas{
    \E_{\vx \sim \calD}\bbb{\exp\bb{\vu^{\top}\vx}} \leq \exp\bb{\frac{\sigma^{2}\norm{\vu}_{2}^{2}}{2}}.
}
\end{definition}

We require the following standard facts to bound the approximation error of random sampling.

\begin{lemma}\label{lem:wedin_projection_step}
Let $\ma,\mb \in\Sym^{d\times d}$, such that
$\gap\defeq \lambda_1(\mb)-\lambda_2(\mb) > 0$. Let
$\vv\defeq \vv_1(\ma)$ and $\vu\defeq \vv_1(\mb)$. Then,
\[
    \normsop{\vv\vv^\top-\vu\vu^\top}
    \le
    \frac{2\normsop{\ma-\mb}}{\gap}.
\]
\end{lemma}
\begin{proof}
Recall the identities
\begin{equation}\label{eq:sin_equiv}
\normsop{\vv\vv^\top - \vu\vu^\top} = \sin \angle(u, v) = \normsop{(\id_d - \vu\vu^\top)\, \vv\vv^\top}
\end{equation}
for any unit vectors $\vu$, $\vv$.
The proof then follows from Corollary~1 of \cite{yu2015useful}, applied with
$\msig\leftarrow \mb$, $\widehat\mb\leftarrow \ma$, and $j=1$.
The denominator is
$\min(\lambda_0(\mb)-\lambda_1(\mb),\lambda_1(\mb)-\lambda_2(\mb))
=\lambda_1(\mb)-\lambda_2(\mb)$, since $\lambda_0(\mb)=+\infty$ in their convention.
\end{proof}

\begin{fact}[Lemma 6.26, \cite{wainwright2019high}]\label{fact:entrywise_err}
Let $\calD$ be $\sigma$-sub-Gaussian with covariance $\msig \in \PSD^{d \times d}$, let $\{\vx_i\}_{i \in [n]} \simiid \calD$, and let $\hmsig \defeq \frac 1 n \sum_{i \in [n]} \vx_i\vx_i^\top$. There exists a universal constant $C > 0$ such that for all $\delta \in (0, 1)$,
\[\Pr\Brack{\max_{(i, j) \in [d] \times [d]} \Abs{\hmsig_{ij} - \msig_{ij}} \ge C\sigma^2\Par{\sqrt{\frac{\log(\frac d \delta)}{n}} + \frac{\log(\frac d \delta)}{n}}}\le \delta.\]
\end{fact}

\begin{lemma}\label{lem:opnorm_err_sparse}
Let $\calD$ be $\sigma$-sub-Gaussian with covariance $\msig \in \PSD^{d \times d}$, $\{\vx_i\}_{i \in [n]} \simiid \calD$, $\hmsig \defeq \frac 1 n \sum_{i \in [n]} \vx_i\vx_i^\top$, and $s \in [d]$. There exists a universal constant $C > 0$ such that for all $\delta \in (0, 1)$,
\[\Pr\Brack{\max_{\substack{S \subseteq [d] \\ |S| \le s}} \normop{\hmsig_{S \times S} - \msig_{S \times S}} \ge C\sigma^2\Par{\sqrt{\frac{s \log(d) + \log(\frac 1 \delta)}{n}}  + \frac{s \log(d) + \log(\frac 1 \delta)}{n}}} \le \delta.\]
\end{lemma}
\begin{proof}
Observe that there are $\binom{d}{s} \le d^s$ such submatrices. Thus, it is enough to apply Exercise 4.7.3, \cite{Vershynin18} to each submatrix with the dimension set to $s$ and the failure probability set to $\delta \gets \delta/d^s$, and then conclude by a union bound.
\end{proof}

We will often use the following truncation error bound, whose proof is deferred to Appendix~\ref{appendix:utility_results}.

\begin{restatable}{lemma}{restatethreshold}\label{lem:kRCS_thresholding} 
Let $\rho \ge \tau \ge 0$, and assume $\ma, \mb \in \Sym^{d \times d}$ satisfy $\normsinf{\mb - \ma} \le \tau$. 
\begin{enumerate}
    \item If $\ma$ is $\kRCS$, then
\[\normop{\thresh_{\rho}(\mb) - \ma} \le 2k \, \rho.\]
    \item If $\nnz(\ma) \le s$, then
\[\normop{\Top_{s}(\mb) - \ma} \le 2\sqrt{2s} \, \rho.\]
\end{enumerate}
\end{restatable}

\textbf{Differential privacy}. Let $\Gamma$ be the domain of the data and $\calD \in \Gamma^n$ be a dataset of $n$ elements. We say that two datasets $\calD$, $\calD' \in \Gamma^n$ are \emph{neighboring} if their symmetric difference has size $1$, i.e., they differ in a single element. We use the following definition of differential privacy.

\begin{definition}[Differential privacy]\label{def:dp}
Let $(\eps, \delta) \in [0, 1]^2$.\footnote{In principle, the privacy parameter $\eps$ can be larger than $1$. However, for any $\eps \ge 1$, our sample-complexity bounds are unaffected up to constant factors if we instead guarantee $(1,\delta)$-DP (which is weaker than $(\eps,\delta)$-DP). Therefore, for convenience and to simplify several bounds, we state all results for $\eps \in [0,1]$.}
We say that a randomized algorithm $\alg: \Gamma^n \to \Omega$ satisfies \emph{$(\eps, \delta)$-differential privacy} (or, is \emph{$(\eps, \delta)$-DP}) if for all events $\event \subseteq \Omega$ and all neighboring datasets $\calD, \calD' \in \Gamma^n$,
\[\Pr\Brack{\alg(\calD) \in \event} \le \exp(\eps) \Pr\Brack{\alg(\calD') \in \event} + \delta.\]
\end{definition}

Differential privacy obeys \emph{basic composition} (Theorem B.1, \cite{dr14}): if $\alg_1:\Gamma^n\to\Omega_1$ is $(\eps_1,\delta_1)$-DP and $\alg_2:\Gamma^n\times \Omega_1\to\Omega_2$ is $(\eps_2,\delta_2)$-DP, then the procedure that runs $\alg_1$ on $\calD$ and subsequently runs $\alg_2$ on $(\calD,\alg_1(\calD))$ is $(\eps_1+\eps_2,\delta_1+\delta_2)$-DP. 

We next state the Gaussian mechanism. Recall that if $\vv: \Gamma^n \to \R^k$ is a vector-valued function of a dataset, we say $\vv$ has sensitivity $\Delta$ if for all neighboring $\calD, \calD' \in \Gamma^n$, we have $\norms{\vv(\calD) - \vv(\calD')}_2 \le \Delta$.

\begin{fact}[Theorem A.1, \cite{dr14}]\label{fact:gaussian_mech}
Let $\vv: \Gamma^n \to \R^k$ have sensitivity $\Delta$, and let $(\eps, \delta) \in [0, 1]^2$. Then, the algorithm that outputs a sample from $\Nor(\vv(\calD), \sigma^2 \id_k)$ is $(\eps, \delta)$-DP for any $\sigma\ge \frac{2\Delta}{\eps} \cdot \sqrt{\log(\frac 2 \delta)}$.
\end{fact}

We denote the bounded Laplace distribution with parameters $\lam, \tau \ge 0$ by $\BL(\lam, \tau)$, which is the distribution of $X \sim \Lap(\lam)$ conditioned on $|X| \le \tau$. We also require the bounded Laplace mechanism, which is known to give the following guarantee.

\begin{fact}[Lemma 9, \cite{AsiLT24}]\label{fact:bl_mech}
Let $s: \Gamma^n \to \R$ have sensitivity $\Delta$, and let $(\eps, \delta) \in [0, 1]^2$. Then, the algorithm that outputs $s(\calD) + \xi$, where $\xi \sim \BL(\frac \Delta \eps, \tau)$, is $(\eps, \delta)$-DP for any $\tau \ge \frac \Delta \eps \log(\frac 4 \delta)$.
\end{fact}

Fact~\ref{fact:bl_mech} is established in \cite{AsiLT24} via a coupling argument, leveraging that $\BL(\lambda)$ and $\Lap(\lambda)$ produce identical samples except with probability $\delta$.

\section{Private Sparse Covariance Estimation}
\label{ssec:priv_sparse_cov_est}

In this section, we provide our main results for solving Problem~\ref{prob:cov_est} (private covariance estimation) under Model~\ref{model:general_krcs_cov}. In Section~\ref{sssec:priv_sparse_cov_est_upper_bound}, we prove Theorem~\ref{thm:priv_sparse_cov_est_upper_bound}, which shows that $\approx \sqrt{d}$ samples (suppressing other parameter dependences) suffice for private covariance estimation, improving upon the $d$ dependence from prior work \cite{wang2021differentially}. After some preprocessing, our algorithm (Algorithm~\ref{alg:oneshot_topk_cov}) simply applies a private top-$k$ selection step to every row of the empirical covariance, and then releases a noised-and-symmetrized variant of the selected entries. In Section~\ref{sssec:priv_sparse_cov_est_lower_bound}, we prove Theorem~\ref{thm:priv_sparse_cov_est_lower_bound}, a complementary lower bound that shows this $\sqrt d$ dependence is necessary up to logarithmic factors.



\subsection{Upper bound}\label{sssec:priv_sparse_cov_est_upper_bound}


\begin{algorithm}[t]
\caption{$\kRCSCOV(\calD, k, \eps, \delta, R)$}
\label{alg:oneshot_topk_cov}
{\bf Input}: Dataset $\calD=\{\vx_t \in \R^d\}_{t\in[n]}$, sparsity $k\in[d]$, privacy $(\eps,\delta) \in (0, 1)^2$, truncation level $R > 0$
\begin{algorithmic}[1]
\State $\vy_t \gets \sign(\vx_t) \circ \min\{|\vx_t|, R\}$ for all $t \in [n]$\label{line:truncate}

\Comment{Thresholding entries: $\min\{\cdot, R\}$, $|\cdot|$, and $\sign(\cdot)$ all entrywise, $\circ$ is entrywise multiplication.}

\State $\hmsig \gets \frac 1 n \sum_{t \in [n]} \vy_t\vy_t^\top$\label{eq:empirical_truncated_cov}
\State $\deltarow\gets \frac{\delta}{2d}$,
$\epsrow \gets \frac \eps 4 \cdot (2d\log(\frac 2 \delta))^{-1/2}$, $\Delta \gets \frac{2R^2}{n}$
\State $b \gets \frac{2\Delta}{\epsrow} \cdot \sqrt{k\log(\frac d {\deltarow})}$
\For{$i \in [d]$}
    \State $S_i \gets \supp(\Top_k(\hmsig_{i,:}^\top + \vw_i))$ where $\vw_i \in \R^d$ has entries $[\vw_i]_j \simiid \Lap(b)$
    \State $\tmsig_{i,:}^\top \gets \hmsig_{i,:}^\top + \vz_i$ where $\vz_i \in \R^d$ has entries $[\vz_i]_j \simiid \Lap(b)$
\EndFor
\State $\mm \gets \mzero_{d\times d}$
\For{$i \in [d]$} 
    \State $\mm_{ii} \gets \tmsig_{ii}$
\EndFor
\For{$1\le i<j\le d$}
    \If{$j\in S_i$ {\bf and} $i\in S_j$}
        \State $\mm_{ij} \gets \frac12(\tmsig_{ij}+\tmsig_{ji})$ and $\mm_{ji} \gets \frac12(\tmsig_{ij}+\tmsig_{ji})$
    \EndIf
\EndFor
\State \Return $\mm$
\end{algorithmic}
\end{algorithm}


\begin{theorem}
\label{thm:priv_sparse_cov_est_upper_bound}
Algorithm~\ref{alg:oneshot_topk_cov} solves Problem~\ref{prob:cov_est}, for $\calD = \{\vx_t\}_{t \in [n]}$ drawn from Model~\ref{model:general_krcs_cov}, with
\[n = \Omega\Par{\frac{k^{2}}{\alpha^{2}}\log\bb{\frac{d}{\beta}} + \frac{k\sqrt{dk}}{\alpha\eps}\log\bb{\frac{d}{\beta}}\log\bb{\frac{d}{\delta}}\log\bb{\frac{nd}{\beta}}}\]
for a sufficiently large constant.
\end{theorem}
\begin{proof}
Throughout, we use the parameter setting
\[R \gets \sigma\sqrt{2\log\bb{\frac{6nd}{\beta}}}\]
in Algorithm~\ref{alg:oneshot_topk_cov}, where $\beta$ is the failure probability from Problem~\ref{prob:cov_est}, and $\sigma$ is the sub-Gaussian parameter from Model~\ref{model:general_krcs_cov}.
We prove privacy and utility of Algorithm~\ref{alg:oneshot_topk_cov} separately.

\paragraph{Privacy.} The truncation in Line~\ref{line:truncate} ensures that for adjacent $(\calD, \calD')$ differing in the $t^{\text{th}}$ entry,
\[
\Abs{\hmsig_{ij}(\calD)-\hmsig_{ij}(\calD')}
= \frac1n\Abs{[\vy_t]_i[\vy_t]_j-[\vy'_t]_i[\vy'_t]_j}
\le \frac{2R^2}{n}=\Delta.
\]
Hence, entries of $\hmsig$ are $\Delta$-sensitive. By Theorem 2.2 of \cite{qiao2021oneshot} and the definition of the Laplace parameter $b$, each release of $S_i$ and the entries of $\tmsig_{i,:}$ (restricted to $S_i$) is $(\epsrow, \deltarow)$-DP. Applying advanced composition (e.g., Theorem 3.20, \cite{dr14} with $\delta' \gets \frac \delta 2$) now yields $(\eps, \delta)$-DP, because
\[\eps^\star \le \epsrow\sqrt{2d\log\frac 2 \delta}+2d\epsrow^2 \le \eps,
\quad
d\deltarow+\frac \delta 2=\delta.\]
The output $\mm$ is a postprocessing of $\tmsig$ restricted to the selected entries, so Algorithm~\ref{alg:oneshot_topk_cov} is $(\eps, \delta)$-DP.
\paragraph{Utility.} We begin by bounding the failure probability of the events
\begin{align*}
&\calE_{\rm clip}\defeq\Big\{\max_{t\in[n]}\max_{j\in[d]}|[\vx_t]_j|\le R\Big\}, \quad \calE_{\rm samp}\defeq\Big\{\max_{(i,j)\in[d]\times[d]} \Abs{ \hmsig_{ij}-\msig_{ij} } \le \lamsamp\Big\},\\
&\calE_{\rm noise}\defeq\Big\{\max_{(i,j)\in[d]\times[d]}|[\vz_i]_j| \le \lampriv\Big\}
\cap \Big\{\max_{(i,j)\in[d]\times[d]}|[\vw_i]_j|\le \lampriv\Big\},
\end{align*}
where
\begin{equation}\label{eq:lamdef}
\lamsamp \defeq C\sigma^2\Par{\sqrt{\frac{\log(6d/\beta)}{n}}+\frac{\log(6d/\beta)}{n}},
\quad
\lampriv \defeq b\log\Par{\frac{12d^2}{\beta}},
\end{equation}
and $C > 0$ is the constant from Fact~\ref{fact:entrywise_err}.
Let \[\calE \defeq \calE_{\rm clip}\cap\calE_{\rm samp}\cap\calE_{\rm noise}\]
be the event that (i) no truncation occurs, (ii) $\hmsig$ concentrates entry-wise, and (iii) all sampled Laplace random variables are bounded. 

We claim that each of the three events in the definition of $\calE$ has failure probability at most $\frac \beta 3$. For  $\calE_{\rm clip}$, this follows from sub-Gaussianity applied to each of $nd$ possible entries of an $\vx_t$. For $\calE_{\rm samp}$, this is immediate from Fact~\ref{fact:entrywise_err}. Finally, for $\calE_{\rm noise}$, this follows from a union bound over $2d^2$ Laplace random variables. We henceforth condition on $\calE$, which gives the failure probability of $\beta$.

Under $\calE$, we have the following bounds uniformly for all $(i, j) \in [d] \times [d]$:
\begin{equation}\label{eq:two_uniform_bounds}
\Abs{\hmsig'_{ij}-\msig_{ij}}\le \lambda
\qquad\text{and}\qquad
\Abs{\tmsig_{ij}-\msig_{ij}}\le \lambda,
\end{equation}
where
$\hmsig'$ denotes the matrix that adds $\vw_i$ to each row $i$ of $\hmsig$, and where $\lam \defeq \lamsamp + \lampriv$. We now claim that we have the entrywise bound
\begin{equation}\label{eq:entrywise_bound_msig}\Abs{\mm_{ij} - \msig_{ij}} \le 2\lam \text{ for all } (i, j) \in [d] \times [d].\end{equation}
To prove \eqref{eq:entrywise_bound_msig}, 
first fix a row $i$. We claim that any 
\begin{equation}\label{eq:not_selected_small}j \not\in S_i \implies |\msig_{ij}| \le 2\lam.\end{equation}
This is because if $|\msig_{ij}| \ge 2\lam$, then by applying \eqref{eq:two_uniform_bounds}, the only entries in $\hmsig'_{i,:}$ that could have larger magnitude are the other nonzero entries in $\msig_{i,:}$, of which there are at most $k$. 

Now, \eqref{eq:entrywise_bound_msig} is immediate for any $(i, j) \in [d] \times [d]$ where $i \not\in S_j$ or $j \not\in S_i$, because then $\mm_{ij} = 0$ and it suffices to apply \eqref{eq:not_selected_small} and symmetry of $\msig_{ij}$. Further, if $i \in S_j$ and $j \in S_i$, then using the second bound in \eqref{eq:two_uniform_bounds} (for the tuples $(i, j)$ and $(j, i)$) and the triangle inequality, \eqref{eq:entrywise_bound_msig} again follows.

Hence, $\mm - \msig$ is entrywise dominated in absolute value by $2\lam \mb$, where $\mb \in \{0, 1\}^{d \times d}$ has $(i, j)^{\text{th}}$ entry $\ind([\mm - \msig]_{ij} \neq 0)$. Because $\mm$ is $(k + 1)$-RCS and $\msig$ is $\kRCS$, $\mb$ is $3k\text{-RCS}$, and therefore
\[\normop{\mm - \msig} \le 2\lam\normop{\mb} \le 6k\lam \le \alpha\sigma^2.\]
The first inequality above used that taking absolute values of a matrix only increases its operator norm (see \eqref{eq:entrywise_abs}), and entrywise-larger nonnegative matrices have larger operator norms due to the Perron-Frobenius theorem. The second inequality used that the operator norm of a $\kRCS$ matrix is $\le k$ (Problem 5.6.P21 in \cite{HornJ12}), and the third used our choice of $n$ and definitions in \eqref{eq:lamdef}.
\end{proof}

\subsection{Lower bound}\label{sssec:priv_sparse_cov_est_lower_bound}
In this section, we prove our lower bound in Theorem~\ref{thm:priv_sparse_cov_est_lower_bound}. Our lower bound is based on a \emph{fingerprinting} strategy, adapted from \cite{narayanan2024better}. We specifically instantiate this strategy using a construction based on the \emph{Inverse Wishart} distribution. We recall some properties of this distribution.

\begin{fact} \label{fact:inv_wishart_facts} Let $\msig \sim \InvWishart\Pars{\mPsi, \nu}$, where $\mPsi \in \PSD^{d \times d}$, $\nu > d + 1$, and let $\calX := \{\vx_{i}\}_{i \in [n]} \simiid \mathcal{N}\Pars{0, \msig}$. Then, the density of $\msig$ is $\propto \det\Pars{\msig}^{-\half\Pars{\nu + d + 1}}\exp\Pars{-\half\Tr\Pars{\mPsi\msig^{-1}}}$, and letting $\hmsig \defeq \frac 1 n \sum_{i \in [n]} \vx_i\vx_i^\top$,
\[\E\Brack{\msig} = \frac 1 {\nu - d - 1} \mPsi,\quad \E\Brack{\msig \mid \calX} = w\hmsig + (1 - w)\E\Brack{\msig},\text{ where } w \defeq \frac{n}{\nu + n - d - 1}.\]
\end{fact}

We also define a \emph{graph projection} operation, used to define scores in the fingerprinting framework. For $G := \Pars{V, E}$ an undirected graph with $V \equiv [d]$, we define the operator $P_{G} := \Sym^{d \times d} \rightarrow \Sym^{d \times d}$ as
\ba{
    [P_{G}\bb{\mm}]_{ij} := \begin{cases}
        \mm_{ij} & i = j \text{ or } (i,j) \in E \\
        0 & \text{otherwise.}
    \end{cases}. \label{def:graph_prof}
}
We are now ready to define our lower bound instance.

\begin{construction}\label{const:fingerprinting_prior} Let $(k, d) \in \N^2$ with $k \mid d$ and $k \ge 2\log(\frac d k)$, and let $B \defeq \frac d k$. Let $G = (V, E)$ with $V \equiv [d]$ be a union of $B$ cliques (with no edges between cliques), so that the $b^{\text{th}}$ clique is on vertices $j$ with $k(b-1) + 1 \le j \le kb$. Draw independent $\{\msig_b \in \PSD^{k \times k}\}_{b \in [B]}$ via
\ba{
\msig_b \simiid \InvWishart\Par{(k - 1)\id_k, 2k},
\label{eq:blockwise_prior_def_model}
}
and let $\msig \in \PSD^{d \times d}$ be block diagonal with $\{\msig_b\}_{b \in [B]}$ along the diagonal, and $\{\vx_i\}_{i \in [n]} \simiid \calN(\0_d, \msig)$.
\end{construction}

Note that for $\msig$ arising from Construction~\ref{const:fingerprinting_prior}, drawing samples from $\calN(\0_d, \msig)$ is an instance of Model~\ref{model:general_krcs_cov} with $\sigma^2 \defeq \normop{\msig}$, because $\msig$ is always $\kRCS$. We next provide some additional useful properties of $\msig$ drawn from Construction~\ref{const:fingerprinting_prior}. As several of these properties follow from small modifications to arguments in \cite{narayanan2024better}, we defer a proof to Appendix~\ref{appendix:priv_sparse_cov_est}.

\begin{restatable}{lemma}{invwishprop}\label{lem:wishart_dist_prop} 
Let $\msig$ be generated as in Construction~\ref{const:fingerprinting_prior} with associated graph $G$. For all $i \in [n]$ and $b \in [B]$, let $\vx_i^b \in \R^k$ denote the coordinates $j \in [d]$ of $\vx_i$ with $k(b-1) + 1 \le j \le kb$. Then the following hold for  universal positive constants $g_1< g_2$, and
\[\hmsig \defeq \frac 1 n \sum_{i \in [n]}\vx_i\vx_i^\top,\quad \hmsig_b := \frac{1}{n}\sum_{i\in [n]}\vx_i^b(\vx_i^b)^{\top},\text{ for all } b \in [B].\]
\begin{enumerate}
\item $P_G(\hmsig)$ is block diagonal with $\{\hmsig_b\}_{b \in [B]}$ along the diagonal.\label{item:cov_prop_1}
    \item For all $i \in [n]$ and $b\in [B]$, $\vx^b_i \simiid \calN(0, \msig_{b})$, and $\vx^b_i \independent \vx^{b'}_i$ for all $b \neq b'$, conditional on $\msig$.\label{item:cov_prop_2}
    \item For all $t > 0$, $\Prob[\normsop{\msig} > t] \leq \frac{d}{k}(\frac{3e^{2}}{t})^{a}$ where $a \defeq \frac{k}{2}$, and if $q \leq \frac{a}{2}$, then $\E[\normsop{\msig}^{q}] \leq 2\exp(6q)$.\label{item:cov_prop_3}
    \item For all $b \in [B]$, $g_{1}\frac{k^{2}}{n} \leq \E\big[\normsf{\hmsig_b-\msig_b}^2\big] \leq g_{2}\frac{k^{2}}{n}$, and $g_{1}\frac{dk}{n} \leq \E[\normsf{P_{G}(\hmsig) - \msig}^{2}] \leq g_{2} \frac{dk}{n}$.\label{item:cov_prop_4}
\end{enumerate}
\end{restatable}

Our proof adapts the fingerprinting-based strategy proposed in \cite{narayanan2024better} to $\kRCS$ matrices. To ease our exposition, we isolate the Bayesian fingerprinting argument of \cite{narayanan2024better} into a self-contained result in Lemma~\ref{lem:bayesian_fingerprinting}, applied in this section and in Appendix~\ref{appendix:dense_pca_lower_bound}. This result operates under the following \emph{Bayesian replacement model}, instantiated appropriately in applications of Lemma~\ref{lem:bayesian_fingerprinting}.
 
\begin{model}[Bayesian replacement model]\label{model:bayesian_fingerprinting}
Let $\Theta$ and $\Gamma$ be measurable spaces, let $\Pi$ be a \emph{prior} distribution on $\Theta$, and let
$\{P_\theta:\theta\in\Theta\}$ be \emph{sample} distributions on $\Gamma$. Draw $\theta\sim\Pi$ and, conditional
on $\theta$, draw $\{X_i\}_{i \in [n]}, \{X'_i\}_{i \in [n]}\simiid P_\theta$. 
Set $\calX \defeq \{X_i\}_{i \in [n]}$, and for all $i \in [n]$, set
\[\calX^{\sim i}\defeq\Brace{X_1,\ldots,X_{i-1},X_i',X_{i+1},\ldots,X_n}.\]
\end{model}

We next state our general framework for establishing approximate DP lower bounds, Lemma~\ref{lem:bayesian_fingerprinting}, which applies whenever a problem-specific
score has a large aggregate signal but DP limits the contribution of each sample. The score $Z_i$ measures the correlation contributed by the observed sample $X_i$, whereas its mean-zero
counterpart $Z_i'$ evaluates $X_i$ against an output that did not observe $X_i$; by DP, $Z'_i$ and $Z_i$ cannot differ by too much. The lower bound then follows by comparing with a lower bound on the total score. We defer the proof to Appendix~\ref{appendix:priv_sparse_cov_est}.

\begin{restatable}{lemma}{bayesianfingerprinting}
\label{lem:bayesian_fingerprinting}
In Model~\ref{model:bayesian_fingerprinting}, let $p \in \N$, let $g:\Theta\to\R^p$ and $\psi:\Theta\times\Gamma\to\R^p$ be measurable, and let
$\calA:\Gamma^n\to\R^p$ be $(\eps,\delta)$-DP for $(\delta, \eps) \in (0, 1)^2$. For all $i\in[n]$, define
\[
Z_i\defeq\inprod{\calA(\calX)-g(\theta)}{\psi(\theta,X_i)},
\quad
Z_i'\defeq\inprod{\calA(\calX^{\sim i})-g(\theta)}{\psi(\theta,X_i)}.
\]
Suppose that for all $i \in [n]$, $\E[\psi(\theta,X_i)\mid\theta]=0$ and, for some $L,U>0$, that
\[
\E\Brack{\sum_{i\in[n]} Z_i}\geq L,
\quad
2\eps\E\Abs{Z_i'}
+2\sqrt{\delta}\sqrt{\E[Z_i^2]+\E[(Z_i')^2]}
\leq U.
\]
Then $n \geq L/U$.
\end{restatable}

Before instating Lemma~\ref{lem:bayesian_fingerprinting}, we require two additional ingredients. The first is a posterior concentration result, adapted from Lemma 4.7 in \cite{narayanan2024better}, which is used to obtain our lower bound $L$ for Lemma~\ref{lem:bayesian_fingerprinting}. We defer its proof to Appendix~\ref{appendix:priv_sparse_cov_est}.

\begin{restatable}{lemma}{restatepostconc}\label{lem:blockwise_posterior_conc} Let $\msig$, $\calX := \{\vx_{i}\}_{i\in[n]}$ be generated from Construction~\ref{const:fingerprinting_prior} and $\hmsig := \frac 1 n\sum_{i \in [n]}\vx_i\vx_i^{\top}$. Then,
\bas{
    \E\bbb{\normsf{\E\bbb{\msig \mid \calX} - P_G(\hmsig)}^{2}} = O\bb{\frac{dk^{2}}{n^2} + \frac{dk^{3}}{n^{3}}}.
}
\end{restatable}

The second is a procedure that takes a covariance estimation algorithm with an expected squared Frobenius norm guarantee for Construction~\ref{const:fingerprinting_prior} under a good event, and boosts it to have a stronger error guarantee. This procedure simply takes the coordinatewise median of logarithmically-many fresh calls to the base estimation algorithm, and applies the tail behavior of medians. 

\begin{lemma}[Lemma A.2, \cite{narayanan2024better}]\label{lem:boost_moment}
Suppose there is an $(\eps, \delta)$-DP algorithm $\calA$ that uses $n$ samples $\calX$ from Model~\ref{model:general_krcs_cov} with $\msig$ drawn from Construction~\ref{const:fingerprinting_prior} for sufficiently large $d$, and satisfies\footnote{The statement in \cite{narayanan2024better} has a constant of $10$, but the proof holds for any larger constant upon examination.} 
\[\Pr\Brack{\normf{\calA(\calX) - \msig} \le \rho} \ge \frac 2 3 \text{ if } \normop{\msig} \le 30,\]
where for a universal constant $c$, $\exp(-ck) \le \rho \le c\sqrt{d}$.
Then there is an $(\eps, \delta)$-DP algorithm $\calA'$ that uses $O(n\log \frac d \rho)$ samples from Model~\ref{model:general_krcs_cov} with $\msig$ drawn from Construction~\ref{const:fingerprinting_prior}, and satisfies
\[\E\Brack{\normf{\calA'(\calX) - \msig}^4} = O(\rho^4).\]
\end{lemma}

We are now ready to conclude our lower bound proof.

\begin{theorem}\label{thm:priv_sparse_cov_est_lower_bound} 
Let $\alpha_0$ be a sufficiently small universal constant, and let $\calA: (\R^d)^n \to \R^d$ be an $(\eps, \delta)$-DP algorithm that, under Model~\ref{model:general_krcs_cov} with $\gamma = 0$, $k \ge \frac 1 {\alpha_0}\log(d)$, and sufficiently large $d$, solves Problem~\ref{prob:cov_est} with $\alpha \le \alpha_0$ and $\beta = \frac 1 3$. Then if $\delta \le \frac{\eps^2}{d^2}$,
\[n = \Omega\Par{\frac{k^2}{\alpha^2} + \min\Brace{ \frac{k\sqrt{d}}{\alpha\eps\log\frac d \alpha},\frac{\sqrt{d}\exp(\alpha_0 k)}{\eps}}}.\]
\end{theorem}
\begin{proof}
The non-private component, $n = \Omega(\frac{k^2}{\alpha^2})$, is a consequence of Theorem 1, \cite{cai2012sparsecov}.

For the other component, let $\rho \defeq 30\alpha \sqrt{d}$, and suppose that $\alpha_0$ is chosen small enough such that $\rho \le c\sqrt{d}$ for the constant $c$ from Lemma~\ref{lem:boost_moment}. We split into two regimes. First, suppose that $\rho \ge \exp(-ck)$. Then by applying the bound from Lemma~\ref{lem:boost_moment}, there is an $(\eps, \delta)$-DP algorithm $\calA'$ that uses $N \defeq O(n \log \frac d \alpha)$ samples $\calX$ and satisfies
\begin{equation}\label{eq:err_assume}\E\Brack{\normf{\calA'(\calX) - \msig}^4} = O(\alpha^4 d^2).\end{equation}
We now use this assumption to instate Lemma~\ref{lem:bayesian_fingerprinting}. Concretely, in Lemma~\ref{lem:bayesian_fingerprinting}, take $n \gets N$, $\theta \gets \msig$, $\Pi \gets$ the prior distribution from Construction~\ref{const:fingerprinting_prior}, $P_\theta \gets \calN(\0_d, \msig)$, $g(\msig) = \msig$, and $\psi(\msig, \vx) \defeq P_G(\vx\vx^\top) - \msig$. We begin with the lower bound $L$ in Lemma~\ref{lem:bayesian_fingerprinting}:
\begin{align*}
\sum_{i \in [N]} Z_i &= \sum_{i \in [N]} \inprod{\calA'(\calX) - \msig}{P_G(\vx_i\vx_i^\top) - \msig} \\
&= N \inprod{\calA'(\calX) - \msig}{\bmsig - \msig},\text{ where } \bmsig \defeq P_G\Par{\hmsig},\; \hmsig \defeq \frac 1 N \sum_{i \in [n]} \vx_i\vx_i^\top.
\end{align*}
Moreover,
\begin{align*}
\E\Brack{\inprod{\calA'(\calX) - \msig}{\bmsig - \msig}} &= \E\Brack{\normf{\bmsig - \msig}^2} + \E\Brack{\inprod{\calA'(\calX) - \bmsig}{\bmsig - \msig}} \\
&= \E\Brack{\normf{\bmsig - \msig}^2} + \E\Brack{\E\Brack{\inprod{\calA'(\calX) - \bmsig}{\bmsig - \msig} \mid \calX}} \\
&= \E\Brack{\normf{\bmsig - \msig}^2} + \E\Brack{\inprod{\E[\calA'(\calX) \mid \calX] - \bmsig}{\bmsig - \E[\msig \mid \calX]}} \\
&\ge \E\Brack{\normf{\bmsig - \msig}^2} - \E\Brack{\normf{\E[\calA'(\calX) \mid \calX] - \bmsig}\normf{\bmsig - \E[\msig \mid \calX]} } \\
&\ge \E\Brack{\normf{\bmsig - \msig}^2} - \sqrt{\E\normf{\calA'(\calX) - \bmsig}^2 \E\normf{ \E\Brack{\msig \mid \calX}- \bmsig}^2}.
\end{align*}
The third line above used $\calA'(\calX) \independent \msig \mid \calX$, the fourth line used Cauchy-Schwarz after conditioning on $\calX$, and the last line used Cauchy-Schwarz after taking expectation over $\calX$. Finally, 
\begin{align*}
\E\Brack{\normf{\bmsig - \msig}^2} &\ge g_1 \cdot \frac{dk}{N}, \\
\E\normf{\calA'(\calX) - \bmsig}^2 \E\normf{ \E\Brack{\msig \mid \calX}- \bmsig}^2 &\le  2\Par{\E\normf{\calA'(\calX) - \msig}^2 + \E\normf{\bmsig - \msig}^2} \\
&\cdot \E\normf{ \E\Brack{\msig \mid \calX}- \bmsig}^2 \\
&= O\Par{\alpha^2 d + \frac{dk}{N}} \cdot O\Par{\frac{dk^2}{N^2} + \frac{dk^3}{N^3}} \\
&= O\Par{\alpha^2 d} \cdot O\Par{\frac{dk^2}{N^2}} \le \frac{g_1^2}{4} \cdot \frac{d^2k^2}{N^2}.
\end{align*}
The first line applied Item~\ref{item:cov_prop_4}, Lemma~\ref{lem:wishart_dist_prop}, the second line used $(a + b)^2 \le 2(a^2 + b^2)$, the fourth line plugged in our bounds from \eqref{eq:err_assume}, Item~\ref{item:cov_prop_4}, Lemma~\ref{lem:wishart_dist_prop} and Lemma~\ref{lem:blockwise_posterior_conc}, and the last line simplified terms using our assumption that $N \ge n = \Omega(\frac{k^2}{\alpha^2})$, by taking $\alpha_0$ small enough. Combining the above three displays shows that we may take $L = \frac{g_1}{2} \cdot dk$ in our Lemma~\ref{lem:bayesian_fingerprinting} application.

Next, for the upper bound $U$ in Lemma~\ref{lem:bayesian_fingerprinting}, we first have
\begin{align*}
\E\Brack{\Par{Z'_i}^2} &= \E\Brack{\inprod{\calA'(\calX^{\sim i}) - \msig}{P_G(\vx_i\vx_i^\top) - \msig}^2} \\
&\le 2\E\Brack{\normop{\msig}^2\normf{\calA'(\calX^{\sim i}) - \msig}^2} \\
&\le 2\sqrt{\E\Brack{\normop{\msig}^4} \E\Brack{\normf{\calA'(\calX^{\sim i}) - \msig}^4}} = O\Par{\alpha^2 d},
\end{align*}
where the second line used Proposition 3.8, \cite{narayanan2024better}, and the third line applied our bounds from \eqref{eq:err_assume} and Item~\ref{item:cov_prop_3}, Lemma~\ref{lem:wishart_dist_prop}.
Applying Jensen's inequality then gives $\E |Z'_i| = O(\alpha \sqrt d)$. Finally, 
\begin{align*}
\E\Brack{Z_i^2} &\le \sqrt{\E\Brack{\normf{\calA'(\calX) - \msig}^4}\E\Brack{\normf{P_G(\vx_i\vx_i^\top) - \msig}^4}} \\
&= O(\alpha^2 d) \cdot \sqrt{\E\Brack{\normf{\vx_i\vx_i^\top - \msig}^4}} = O(\alpha^2 d^3),
\end{align*}
where we applied \eqref{eq:err_assume} and the bounds in Proposition 3.10 and Eq.\ (10) of \cite{narayanan2024better}. In sum,
\[U = 2\eps \E \Abs{Z'_i} + 2\sqrt{\delta} \sqrt{\E\Brack{Z_i^2} + \E\Brack{\Par{Z'_i}^2}} = O\Par{\eps \alpha \sqrt{d} + \sqrt{\delta} \alpha d^{1.5}} = O(\eps \alpha \sqrt{d}),\]
for our assumed range of $\delta$. Plugging both bounds into Lemma~\ref{lem:bayesian_fingerprinting}, and adjusting by an $\frac N n = O(\log \frac d \alpha)$ factor, gives the first term in the $\min$ in our claimed sample complexity, when $\rho \ge \exp(-cd)$.

Finally, we handle the case when $\rho \le \exp(-ck)$. In this case, it is enough to provide a lower bound when $\rho = 30\alpha \sqrt{d}= \exp(-ck)$ exactly, by monotonicity of the error guarantee in $\alpha$. For this particular $\alpha$, our earlier bound simplifies for sufficiently small $\alpha_0$:
\[\frac{k\sqrt{d}}{\alpha \eps \log \frac d \alpha} = \Omega\Par{\frac{\sqrt{d} \exp(\alpha_0 k)}{\eps}}.\]
\end{proof}

\section{Private Sparse PCA: Algorithms}\label{sssec:priv_sparse_pca_upper_bound}

In this section, we provide our upper bounds for Problem~\ref{prob:pca} under Model~\ref{model:general_krcs_pca}, our strengthening of Model~\ref{model:general_krcs_cov} that posits the leading eigenvector is $k$-sparse. Importantly, our main result (Theorem~\ref{thm:priv_sparse_pca_upper_bound_gamma}) solves Problem~\ref{prob:pca} with a sample complexity scaling \emph{polylogarithmically} in $d$ (and polynomially in other problem parameters), bypassing the curse of dimensionality. The algorithm corresponding to Theorem~\ref{thm:priv_sparse_pca_upper_bound_gamma} is given in Algorithm~\ref{alg:friendly_DP_projector}, and uses the following convenient definition.

\begin{definition}[Goodness]\label{def:good}
We say that $\mm \in \PSD^{d \times d}$ is \emph{$(k, \hlam, \gamma, \tau)$-good} if $\calT_\tau(\mm)$ is $k$-RCS, and
\[\lam_1\Par{\calT_\tau(\mm)} \ge \Par{1 - \gamma}\hlam,\quad \lam_2\Par{\calT_\tau(\mm)} \le \Par{1 - \gamma}\lam_1\Par{\calT_\tau(\mm)}.\]
\end{definition}

Algorithm~\ref{alg:friendly_DP_projector} proceeds in two stages. The first stage (Lines~\ref{line:fc_start} to~\ref{line:fc_end}) computes an aggregated estimate $\bmsig$ as a weighted average over $m$ batches of sample covariances, after certifying that enough batch covariances are good (Definition~\ref{def:good}). Importantly, this guarantees that $\bmsig$ has bounded sensitivity under $\norm{\cdot}_{\infty, \infty}$, using an argument inspired by the $\FC$ framework of \cite{FriendlyCore-tsfadia22a} (more specifically, its recent simplified variants in \cite{lowy2024faster, kumar2025private}). The second stage (Lines~\ref{line:npca_start} to~\ref{line:npca_end}) employs a standard clip-and-PCA procedure on the top eigenspace projector of $\bmsig$, after applying entrywise Gaussian noise to ensure privacy. Our utility analysis (cf.\ Lemma~\ref{lem:utility_2}) combines the $\ell_\infty$ sensitivity of $\bmsig$ with the sparsity assumption in Model~\ref{model:general_krcs_pca} to derive Theorem~\ref{thm:priv_sparse_pca_upper_bound_gamma}.

Finally, throughout Section~\ref{ssec:priv_sparse_pca_utility_analysis}, for a sufficiently large constant $C > 0$, we set
\begin{equation}\label{eq:tau_choice_pca}\tau \defeq C\sigma^2\sqrt{\frac{\log(\frac d \beta)}{b}}\end{equation}
in Algorithm~\ref{alg:friendly_DP_projector}, where $\sigma$ is the sub-Gaussianity parameter from Model~\ref{model:general_krcs_cov}, and $b = \frac n m$ is the batch size. Ultimately, we take $m \approx \frac 1 \eps$ for our privacy analysis. While our privacy proof in Section~\ref{ssec:priv_sparse_pca_priv_analysis} holds for any $\tau$, it may be helpful to keep \eqref{eq:tau_choice_pca} in mind, as this setting forces our sensitivity bounds to scale as $\poly(k, \log(d))$ (e.g., in Lemma~\ref{lem:sensitivity_projector_gamma}), to avoid $n = \Omega(\poly(d))$ sample complexities.

\begin{algorithm}[!hbt]
\caption{\label{alg:friendly_DP_projector}$\kRCSPCA(\calD, \hlam, \gamma, k, \eps, \delta, \tau, m)$}
{\bf Input:} Dataset $\calD = \{\vx_i \in \R^d\}_{i \in [n]}$, operator norm estimate $\hlam > 0$, gap parameter $\gamma > 0$, sparsity parameter $k \in [d]$, privacy parameters $(\eps,\delta) \in (0, 1)^2$, threshold $\tau > 0$, number of batches $m \in \N$
\begin{algorithmic}[1]
\State $\hmsig_i \gets \frac 1 b \sum_{j = (i - 1)b + 1}^{ib} \vx_j\vx_j^\top$ for all $i \in [m]$, where $b \defeq \frac n m$\label{line:fc_start}
\State $\calC \gets \{i \in [m]: \hmsig_i \text{ is } (k, \hlam, \frac \gamma 2, \tau)\text{-good}\}$\label{line:calc_compute}
\State $L \sim \BLap(\frac{3}{\eps}, \frac{3}{\eps}\log(\frac{12}{\delta}))$
\If{$|\mathcal{C}| + L - \frac{3}{\eps}\log(\frac{12}{\delta}) < 0.8m$} \label{eq:check_1_begin}
  \State \Return $\perp$
\EndIf
\For{$i \in [m]$}
  \State $f_i \gets \sum_{j \in [m]} \ind\{\normsinf{\hmsig_{j}-\hmsig_i} \le 2\tau\}$
  \State $p_i \gets \min\{ \max\{\frac{f_i - m/2}{m/6},\, 0\}, 1\}$ \label{line:friendly_core_weighting}
\EndFor \label{eq:check_1_end}
\State $Z \gets \sum_{i \in [m]} p_i$
\State $\xi \sim \BLap(\frac{21}{\eps}, \frac{21}{\eps}\log(\frac{12}{\delta}))$
\If{$Z + \xi - \frac{21}{\eps}\log(\frac{12}{\delta}) < 0.8m$} \label{eq:check_2_begin}
  \State \Return $\perp$
\EndIf \label{eq:check_2_end}
\State $\bar{\msig} \gets \frac{1}{Z}\sum_{i \in [m]} p_i \hmsig_i$\label{line:fc_end}
\State $\vu \gets \vv_{1}\Pars{\thresh_{5\tau}(\bar{\msig})}$\label{line:npca_start}
\State $\hmpp \gets \vu\vu^\top$\label{line:uut}
\State $\sigma_{\rm priv} \gets  \frac{80\sqrt{2}k\tau}{\gamma\hlam} \cdot \frac{6}{\eps}\sqrt{\log\Pars{\frac{6}{\delta}}}$\label{line:sig_set}
\State $\mg \gets $ $d \times d$ matrix with i.i.d.\ entries $\mg_{ij}\sim \Nor(0,\sigma_{\rm priv}^2)$\label{line:gmech}
\State $\widetilde \mpp \gets \Top_{k^2}\Pars{\hmpp+\mg}$ \label{eq:topksquared}
\State \Return $\hvv \defeq \vv_1(\half(\tmpp + \tmpp^\top))$\label{line:npca_end}
\end{algorithmic}
\end{algorithm}

\subsection{Privacy analysis}\label{ssec:priv_sparse_pca_priv_analysis}

In this section, we provide a privacy analysis of Algorithm~\ref{alg:friendly_DP_projector}. Our analysis has two parts: bounding the privacy of the decision to continue executing after running the checks in Lines~\ref{line:fc_start} to~\ref{eq:check_2_end}, and bounding the privacy of the remaining Lines~\ref{line:fc_end} to~\ref{line:npca_end}. We begin with the first (simpler) part.


\begin{lemma}\label{lem:dp_noisy_cert_count_gamma}
The quantities $|\calC| + L$ and $Z + \xi$, on Lines~\ref{eq:check_1_begin} and~\ref{eq:check_2_begin} of Algorithm~\ref{alg:friendly_DP_projector}, are each $(\frac \eps 3, \frac \delta 3)$-DP.
\end{lemma}
\begin{proof}
Each sample $\vx_i$ only affects one block covariance $j \in [m]$, so $|\calC|$ is $1$-sensitive (and similarly each $f_{j'}$ is $1$-sensitive for all $j' \neq j$). We can thus bound the sensitivity of $Z$ by
\ba{
 1 + \bb{m-1} \cdot \frac{6}{m} \leq 7. \label{eq:Z_diff}
}
The claims then follow from Fact~\ref{fact:bl_mech}.
\end{proof}

Our next goal is to establish a sensitivity bound on the projector $\hmpp$ computed on Line~\ref{line:uut}, so that the Gaussian mechanism on Line~\ref{line:gmech} ensures privacy. To do so, we first give a helper lemma on the stability of the weighted average $\bmsig$ on Line~\ref{line:fc_end}, that is used to compute $\hmpp$.

\begin{lemma}\label{lem:sensitivity_barSigma_gamma}
Let $m$ be at least a sufficiently large constant, let $\calD$, $\calD'$ be neighboring datasets, and condition on two runs of Algorithm~\ref{alg:friendly_DP_projector} on $\calD$, $\calD'$ both reaching Line~\ref{line:fc_end}. There exists $\mm \in \PSD^{d \times d}$ such that $\mm$ is $(k, \hlam, \frac \gamma 2, \tau)$-good, and, denoting by $\bmsig$ and $\bmsig'$ the matrices computed on Line~\ref{line:fc_end} by these two runs,
\[\max\Brace{\normop{\calT_{5\tau}\Par{\bmsig} - \calT_\tau(\mm)},\normop{\calT_{5\tau}\Par{\bmsig'} - \calT_\tau(\mm)}} \le 10k\tau.\]
\end{lemma}
\begin{proof}
Let $\calC$ and $\calC'$ denote the sets computed on Line~\ref{line:calc_compute} when Algorithm~\ref{alg:friendly_DP_projector} is run with $\calD$ and $\calD'$ respectively. Similarly, let $\calS$ and $\calS'$ denote the sets of batch indices $i \in [m]$ that have $p_i > 0$ in Line~\ref{line:friendly_core_weighting}, on these two runs of Algorithm~\ref{alg:friendly_DP_projector}. Because the two runs both passed the checks on Lines~\ref{eq:check_1_begin} and~\ref{eq:check_2_begin}, we deterministically have that
\[\min\Brace{|\calC|, |\calC'|, |\calS|, |\calS'|} \ge 0.8m.\]
Next, assume without loss of generality that $\calD$, $\calD'$ differ in the first batch, and view $\calC$, $\calS$ as subsets of $[m]$ and $\calC'$, $\calS'$ as subsets of $[m + 1] \setminus \{1\}$ (i.e., we swap out the first batch in $\calD$ with a new batch with index $m + 1$, equated with the first batch in $\calD'$). We claim that $|\calC \cap \calC' \cap \calS \cap \calS'|$ is nonempty. To see this, since all of $\calC, \calC', \calS, \calS'$ are subsets of $\calU \defeq [m + 1]$,
\begin{align*}
|\calC \cap \calC' \cap \calS \cap \calS'| &= |\calU| - \Abs{\Par{\calU \setminus \calC} \cup \Par{\calU \setminus \calC'} \cup \Par{\calU \setminus \calS} \cup \Par{\calU \setminus \calS'}} \\
&\ge m + 1 - 4(0.2m + 1) \ge 0.2m - 3 \ge 1,
\end{align*}
for sufficiently large $m$. Let $\mm = \hmsig_i$ for the index $i \in \calC \cap \calC' \cap \calS \cap \calS'$. By the definitions of $\calC$, $\calC'$, $\mm$ is $(k, \hlam, \frac \gamma 2, \tau)$-good. Moreover, every batch $i \in \calS$ has that $\ball_\infty(\hmsig_i, 2\tau)$ covers at least $\frac m 2$ elements of $\{\hmsig_i\}_{i \in [m]}$, so that $\ball_\infty(\hmsig_i, 2\tau) \cap \ball_\infty(\mm, 2\tau)$ is nonempty. By the triangle inequality, $\ball_\infty(\mm, 4\tau)$ covers all of the $\{\hmsig_i\}_{i \in \calS}$, and a similar argument applies for $\calS'$. Because $\ball_\infty(\mm, 4\tau)$ is convex, we also have $\bmsig \in \ball_\infty(\mm, 4\tau)$ and $\bmsig' \in \ball_\infty(\mm, 4\tau)$. The conclusion for $\bmsig$ follows from
\[\norm{\bmsig - \calT_\tau(\mm)}_{\infty, \infty} \le \norm{\bmsig - \mm}_{\infty, \infty} + \norm{\mm - \calT_\tau(\mm)}_{\infty, \infty} \le 4\tau + \tau,\]
and then applying Lemma~\ref{lem:kRCS_thresholding} with $\rho \gets 5\tau$. The conclusion for $\bmsig'$ is symmetric.
\end{proof}

\begin{lemma}\label{lem:sensitivity_projector_gamma}
Conditioned on Line~\ref{line:uut} being reached on a run of Algorithm~\ref{alg:friendly_DP_projector}, the matrix $\hmpp$ computed on this line is $\Delta$-sensitive in $\normf{\cdot}$, where
\begin{equation}\label{eq:deltadef}\Delta \defeq \frac{80\sqrt{2} k\tau}{\gamma\hlam}.\end{equation}
\end{lemma}
\begin{proof}
Fix adjacent $\calD$, $\calD'$ and let $\mm$ be the result of Lemma~\ref{lem:sensitivity_barSigma_gamma}. By goodness of $\mm$, $\calT_\tau(\mm)$ has a unique leading eigenvector $\vw$, with associated projector $\mw \defeq \vw\vw^\top$. By Lemma~\ref{lem:sensitivity_barSigma_gamma}, $\normsop{\calT_\tau(\mm) - \calT_{5\tau}(\bmsig)} \le 10k\tau$.
Thus, Lemma~\ref{lem:wedin_projection_step} with $\ma \gets \calT_{5\tau}(\bmsig)$, $\mb \gets \calT_{\tau}(\mm)$, and $\mathsf{gap} \gets \frac {\gamma\hlam} 2$, implies
\[\normop{\hmpp - \mw} \le \frac{40k\tau}{\gamma\hlam}. \]
The same bound symmetrically applies for $\normsop{\hmpp' - \mw}$, and thus because $\hmpp$, $\hmpp'$ are rank-one,
\[\normf{\hmpp -\hmpp'} \le \sqrt{2}\normop{\hmpp - \hmpp'} \le \sqrt{2}\Par{\normop{\hmpp - \mw} + \normop{\hmpp' - \mw}} \le \frac{80\sqrt{2}k\tau}{\gamma \hlam}.\]
\end{proof}


\begin{corollary}\label{cor:privacy_pca_gamma}
Algorithm~\ref{alg:friendly_DP_projector} is $(\eps,\delta)$-DP.
\end{corollary}

\begin{proof}
By Lemma~\ref{lem:sensitivity_projector_gamma}, the Frobenius sensitivity of the map
$\calD\mapsto \hmpp(\calD)$ assuming Line~\ref{line:uut} is reached is at most $\Delta$, as defined in \eqref{eq:deltadef}.
The Gaussian mechanism (Fact~\ref{fact:gaussian_mech}), applied to the vectorization of
$\hmpp$, then yields $(\frac \eps 3, \frac \delta 3)$-DP for releasing $\hmpp+\mg$ for the stated setting of 
$\sigma_{\rm priv}$.
All subsequent steps ($\Top_{k^2}$, symmetrization, and computing $\vv_1(\cdot)$) are postprocessings of this private statistic,
so by using Lemma~\ref{lem:dp_noisy_cert_count_gamma} along with basic composition, the final output is $(\eps,\delta)$-DP.
\end{proof}

\subsection{Utility analysis}\label{ssec:priv_sparse_pca_utility_analysis}

We now analyze the utility of Algorithm~\ref{alg:friendly_DP_projector}, and conclude our proof of Theorem~\ref{thm:priv_sparse_pca_upper_bound_gamma}. To begin, we show that when the data is drawn from Model~\ref{model:general_krcs_cov} for sufficiently large $n$, the algorithm passes the tests in Lines~\ref{eq:check_1_begin} and~\ref{eq:check_2_begin}, and the aggregated $\bmsig$ is close to the population covariance $\msig$.

\begin{lemma}\label{lem:utility_1}
Let $\tau$ be set as in \eqref{eq:tau_choice_pca}, let $\hlam \in [(1 - \frac \gamma {10})\lam_1(\msig), (1 + \frac \gamma {10})\lam_1(\msig)]$, and assume that 
\[m = \Omega\Par{\frac{1}{\eps}\log\Par{\frac{1}{\beta}}},\quad n = \Omega\Par{\frac{\sigma^4}{\lam_1(\msig)^2} \cdot \frac{k^2 m\log(\frac d \beta)}{\gamma^2}}\]
for sufficiently large constants.
Then with probability at least $1 - \frac \beta 4$, if $\calD$ is drawn from Model~\ref{model:general_krcs_cov}, the tests in Lines~\ref{eq:check_1_begin} and~\ref{eq:check_2_begin} pass, $\min\{|\calC|, Z\} \ge 0.9m$, and every $i \in [m]$ with $p_i > 0$ satisfies
\[\norm{\hmsig_i - \msig}_{\infty, \infty} \le 3\tau.\]
\end{lemma}
\begin{proof}
We first prove that Line~\ref{eq:check_1_begin} passes and $|\calC| \ge 0.9m$, with probability $\ge 1 - \frac \beta 4$. We claim that for each batch $i \in [m]$, $\hmsig_i$ is $(k, \hlam, \frac \gamma 2, \tau)$-good with probability $\ge \frac{19}{20}$. If so, a standard Chernoff bound (for a large enough constant in $m$) proves that at least $\frac 9 {10}$ of the $\hmsig_i$ are good, except with probability $\frac \beta4$. In this case, taking $m$ so that $|L| \le \frac m {20}$ deterministically,  Line~\ref{eq:check_1_begin} always passes.

It remains to prove that each $\hmsig_i$ is good with probability $\ge \frac{19}{20}$. By Fact~\ref{fact:entrywise_err}, taking $\tau$ as in \eqref{eq:tau_choice_pca}, with probability $\ge \frac{19}{20}$ we have $\norms{\hmsig_i - \msig}_{\infty, \infty} \le \tau$, and hence no zero entry of $\msig$ is nonzero in $\calT_\tau(\hmsig_i)$. Under this event, Lemma~\ref{lem:kRCS_thresholding} gives
\[\normop{\calT_\tau(\hmsig_i) - \msig} \le 2k\tau \le \frac{\gamma \lam_1(\msig)}{10},\]
where the last inequality holds for a large enough constant in $n$. Under the stated assumption on $\hlam$, the remaining conditions in Definition~\ref{def:good} now hold by Weyl's inequality.

Henceforth, condition on the earlier event that at least $\frac 9 {10}$ of the $i \in [m]$ have $\norms{\hmsig_i - \msig}_{\infty, \infty} \le \tau$; call these indices $G \subseteq [m]$. We next prove that if $|G| \ge 0.9m$, $Z \ge 0.9m$, which implies that Line~\ref{eq:check_2_begin} passes. Indeed, every $i \in G$ has $f_i \ge 0.9m$ by the triangle inequality, and therefore 
\[Z = \sum_{i \in [m]} p_i \ge |G| = 0.9m.\] 
To obtain the last claim, take some $i \in [m]$ with $\norms{\hmsig_i - \msig}_{\infty, \infty} > 3\tau$. By the triangle inequality, for every $j \in G$, $\hmsig_j \not\in \ball_{\infty}(\hmsig_i, 2\tau)$, and hence $f_i \le 0.1m$. Thus, $p_i = 0$ for any such $i$.
\end{proof}

We next prove that under the closeness guarantee afforded by Lemma~\ref{lem:utility_1}, the remaining steps in Algorithm~\ref{alg:friendly_DP_projector} yield a sufficiently good solution to PCA (Problem~\ref{prob:pca}). Notably, this is the step that leverages the stronger eigenvector sparsity assumption from Model~\ref{model:general_krcs_pca}.

\begin{lemma}\label{lem:utility_2}
In the setting of Lemma~\ref{lem:utility_1}, assume that the conclusion of Lemma~\ref{lem:utility_1} holds, and that $\calD$ is drawn from Model~\ref{model:general_krcs_pca}.
Then, with probability $\ge 1 - \frac \beta 4$, denoting $\mpp \defeq \vv\vv^\top$ where $\vv = \vv_1(\msig)$, and $\tmpp$ as in Line~\ref{eq:topksquared} of Algorithm~\ref{alg:friendly_DP_projector},
\[\normop{\tmpp - \mpp} \le 2\sqrt{2}k\Par{\frac{20k\tau}{\gamma\lam_1(\msig)} + 2\sigma_{\mathrm{priv}} \sqrt{\log\Par{\frac{4d}{\beta}}}}.\]
\end{lemma}
\begin{proof}
Under Lemma~\ref{lem:utility_1}'s conclusion, convexity implies $\norms{\bmsig - \msig}_{\infty, \infty} \le 3\tau$, so that Lemma~\ref{lem:kRCS_thresholding} gives
\[\normop{\calT_{5\tau}\Par{\bmsig} - \msig} \le 10k\tau \le \frac{\gamma \lam_1(\msig)}{10},\]
under the sparsity assumption in Model~\ref{model:general_krcs_pca} and our setting of $n$.
Then, Lemma~\ref{lem:wedin_projection_step} with $\mathsf{gap} \gets \gamma \lam_1(\msig)$, $\ma \gets \calT_{5\tau}(\bmsig)$, and $\mb \gets \msig$, implies
\[\normop{\hmpp - \mpp} \le \frac{20k\tau}{\gamma \lam_1(\msig)}.\]
Next, applying $\norm{\cdot}_{\infty, \infty} \le \normop{\cdot}$ and a standard tail bound on the maximum of $d^2$ Gaussian random variables (e.g., Mill's inequality) gives that with probability $\ge 1 - \frac \beta 4$,
\[\norm{\Par{\hmpp + \mg} - \mpp}_{\infty, \infty} \le \frac{20k\tau}{\gamma\lam_1(\msig)} + 2\sigma_{\mathrm{priv}} \sqrt{\log\Par{\frac{8d}{\beta}}}.\]
The conclusion follows from Lemma~\ref{lem:kRCS_thresholding}, because $\mpp$ has at most $s = k^2$ nonzero entries.
\end{proof}

We require one additional ingredient: a private eigenvalue estimation procedure under Model~\ref{model:general_krcs_cov}, whose proof is deferred to Appendix~\ref{appendix:priv_sparse_pca_upper_bound}.

\begin{restatable}{proposition}{restateopnormest}\label{thm:priv_sparse_cov_operator_norm_estimation}
Let $(\eps, \delta, \beta) \in (0,1)^{3}$, and let $\calD$ be drawn from Model~\ref{model:general_krcs_cov} with
\bas{
n = \Omega\Par{\frac{\sigma^4}{\lam_1(\msig)^2} \cdot \frac{k^2\log^4(\frac d {\beta\delta})}{\gamma^2\eps^3}}
}
for a sufficiently large constant.
There is an $(\eps, \delta)$-DP algorithm, $\kRCSOPNORM$ (Algorithm~\ref{alg:friendly_DP_projector_opnorm}), which returns $\hlam$ satisfying $\hlam \in [(1 - \frac \gamma {10})\lam_1(\msig), (1 + \frac \gamma {10}) \lam_1(\msig)]$, with probability $\ge 1-\frac \beta 2$.
\end{restatable}

Finally, we combine the pieces to prove Theorem~\ref{thm:priv_sparse_pca_upper_bound_gamma}.

\begin{theorem}\label{thm:priv_sparse_pca_upper_bound_gamma}
There is an algorithm (Algorithm~\ref{alg:friendly_DP_projector} with $\tau$ set as in \eqref{eq:tau_choice_pca}, and using Proposition~\ref{thm:priv_sparse_cov_operator_norm_estimation} to compute $\hlam$) that solves Problem~\ref{prob:pca}, for $\calD = \{\vx_i\}_{i \in [n]}$ drawn from Model~\ref{model:general_krcs_pca}, with 
\[n =  \Omega\Par{\frac{\sigma^4}{\lam_1(\msig)^2} \cdot \frac{k^4\log^4(\frac d {\beta\delta})}{\gamma^2\Delta \eps^3}},\]
for a sufficiently large constant.
\end{theorem}
\begin{proof}
Throughout the proof, we take $m$ sufficiently large for Lemma~\ref{lem:utility_1} to hold, and set $\tau$ as in \eqref{eq:tau_choice_pca}. Under these settings, condition on the results of Lemma~\ref{lem:utility_1}, Lemma~\ref{lem:utility_2}, and Proposition~\ref{thm:priv_sparse_cov_operator_norm_estimation} all holding, which occurs with probability $\ge 1 - \beta$. Then by convexity of $\normsop{\cdot}$,
\begin{align*}\normop{\half\Par{\tmpp + \tmpp^\top} - \mpp} &\le 2\sqrt{2}k \Par{\frac{20k\tau}{\gamma\lam_1(\msig)} + 2\sigma_{\mathrm{priv}} \sqrt{\log\Par{\frac{4d}{\beta}}}} \\
&\le \frac{2\sqrt{2}k^2 \tau}{\gamma \lam_1(\msig)} \Par{20 + \frac{3600\sqrt 2}{\eps}\log\Par{\frac{6d}{\beta\delta}}}
\le \frac{\sqrt \Delta}{2},\end{align*}
where we plugged in our choice of $\sigma_{\mathrm{priv}}$ from Line~\ref{line:sig_set}, and our choice of $\tau$ from \eqref{eq:tau_choice_pca}, by taking
\[b = \Omega\Par{\frac{\sigma^4}{\lam_1(\msig)^2} \cdot \frac{k^4\log^3(\frac d {\beta\delta})}{\gamma^2\Delta \eps^2}},\quad m = \Omega\Par{\frac 1 \eps \log\Par{\frac 1 {\delta\beta}}}.\]
Our bound on $n$ follows by using the definition $n = mb$. The privacy claim in Problem~\ref{prob:pca} is immediate from Corollary~\ref{cor:privacy_pca_gamma}.
For the utility claim, it follows from Lemma~\ref{lem:wedin_projection_step} with $\ma \gets \half(\tmpp + \tmpp^\top)$, $\mb \gets \mpp$, and $\mathsf{gap} \gets 1$ that $\normsop{\hvv\hvv^\top - \mpp} \le \sqrt{\Delta}$, and the sine-squared error follows from the equivalence \eqref{eq:sin_equiv}.
\end{proof}

\section{Private Sparse PCA: Lower Bounds}\label{sssec:priv_sparse_pca_lower_bound}

In this section, we give dimension-dependent lower bounds for Problem~\ref{prob:pca} under Model~\ref{model:general_krcs_cov}, complementing our upper bound, Theorem~\ref{thm:priv_sparse_pca_upper_bound_gamma}, which uses $\poly(k, \log(d))$ samples under the stronger Model~\ref{model:general_krcs_pca}. We provide two incomparable lower bounds: the first (Section~\ref{ssec:pure_pca}) holds under pure DP, with the same parameterization as used in Theorem~\ref{thm:priv_sparse_pca_upper_bound_gamma}, whereas the second (Section~\ref{ssec:approx_pca}) holds under approximate DP, with a different parameterization commented on in Remark~\ref{rem:ub_compatibility}.

\subsection{Pure DP lower bound}\label{ssec:pure_pca}


In this section, we prove Theorem~\ref{thm:priv_sparse_pca_pure_dp_lower_bound}, our lower bound for PCA for $\kRCS$ covariance matrices (Problem~\ref{prob:pca} under Model~\ref{model:general_krcs_cov}) with \emph{pure DP}. Our argument uses the following well-established \emph{packing lower bound} framework with a careful combinatorial construction.

\begin{proposition}[\cite{kamath2020private}, Lemma 6.2]
\label{prop:priv_lower}
Let $\alpha \in (0,1]$ and $\eps \ge 0$, and let $\calP = \{P_1, P_2, \dots, P_m\}$ be a family of distributions
such that $\mathrm{TV}(P_i, P_j) \le \alpha$ for all disjoint $i,j \in [m]$.
Suppose $\calA: \Gamma \to [m]$ is an $\eps$-DP algorithm such that for all $i \in [m]$,
\[
\Prob_{\vx_1, \dots, \vx_n \simiid P_i,\calA} \Brack{\calA(\vx_1, \ldots, \vx_n) = i} \ge \frac{2}{3}.
\]
Then $n = \Omega(\frac{\log m}{\alpha \eps})$.
\end{proposition}

Our construction for instating Proposition~\ref{prop:priv_lower}, combines two well-known combinatorial ingredients. The first is a variant of the Gilbert-Varshamov bound (see e.g., Theorem 7, \cite{graham2003lower}).

\begin{lemma}
\label{lem:gv_constant_weight}
Let $g \in \N$ and let $\mathcal{S} \defeq \{\vs \in \{0, 1\}^{3g}: \sum_{i \in [3g]} \vs_i = g\}$. There exists a subset $C \subseteq \mathcal{S}$ of size $\exp(\Omega(g))$ such that for any distinct $\vx, \vy \in C$, $\ham(\vx, \vy) \ge \frac g 2$.
\end{lemma}
\begin{proof}
The construction is greedy: repeatedly place any $\vx \in \calS$ in $C$, with Hamming distance at least $\frac g 2$ from every element in $C$, until none exist. For any $\vx \in \calS$, there are at most
\bas{
1 + \sum_{t\in [\frac g 4]} \binom{g}{t} \binom{2g}{t} \le 1 + \frac{g}{4} \binom{g}{ g/4} \binom{2g}{g/4} \le 1 + \frac{g}{4}\cdot 2^{gH_2(1/4)+2gH_2(1/8)} \le \frac{g \cdot 2^{1.9 g}}{3}.
}
binary vectors $\vy \in \calS$ at Hamming distance $\le \frac g 2$ from $\vx$,
where for $p \in (0,1)$, $H_{2}(p) \defeq -p\log(p) - (1-p)\log(1-p)$. Thus, a volume argument yields
\bas{
|\calC| \ge \frac{\binom{3g}{g}}{\frac{g \cdot 2^{1.9 g}}{3}} \ge \frac{3}{g \cdot 2^{1.9 g}} \cdot \frac{0.4}{\sqrt{g}} \bb{\frac{27}{4}}^g \ge \frac{ 4^g}{g^{\frac 3 2}} = \exp\Par{\Omega(g)}.
}
\end{proof}


The second ingredient is the existence of a bipartite expander graph with appropriate spectral characteristics. Our packing instance eventually permutes the vertices of this graph using permutations given by Lemma~\ref{lem:gv_constant_weight}. We first use the following definition.

\begin{definition}\label{def:biregular_bipartite}
A bipartite graph $G = (V,E)$ is said to be $(c_1,c_2)$-biregular if its vertex set can be partitioned as
$V = L \cup R$ such that $\deg(u) = c_1$ for all $u \in L$ and $\deg(v) = c_2$ for all $v \in R$.
\end{definition}

For biregular bipartite graphs, the top eigenpair has an explicit form.

\begin{lemma}
\label{lem:biregular_top_eigen}
Let $G=(L\cup R,E)$ be a $(c_1,c_2)$-biregular bipartite graph, and let $\ma\in\{0, 1\}^{V \times V}$ be its adjacency matrix.
Then $\lam_1(\ma) = \normsop{\ma} =\sqrt{c_1c_2}$, with corresponding eigenvector
\[
\vv \defeq
\frac{1}{\sqrt{c_1|L| + c_2|R|}}
\begin{pmatrix}
\sqrt{c_1}\,\mathbf{1}_{L}\\
\sqrt{c_2}\,\mathbf{1}_{R}
\end{pmatrix}.
\]
\end{lemma}

\begin{proof}
The equality $\ma \vv = \sqrt{c_1 c_2}\vv$ is a direct calculation. To show $\normsop{\ma} = \sqrt{c_1 c_2}$, write
\[\ma=\begin{pmatrix}\mzero & \mb\\ \mb^\top & \mzero\end{pmatrix},\]
where
$\mb\in\{0,1\}^{L \times R}$ is the bipartite incidence matrix. Then $\normsop{\ma}$ is the largest singular value of $\mb$ (this can be seen by e.g., squaring $\ma$). Therefore, the conclusion holds by Cauchy-Schwarz:
\[\sup_{\vx \in \R^R: \norm{\vx}_2 = 1} \norm{\mb \vx}_2^2 = \sup_{\vx \in \R^R: \norm{\vx}_2 = 1}\sum_{\ell \in L} \Par{\sum_{\substack{r \in R \\ (\ell, r) \in E}} \vx_r}^2 \le \sup_{\vx \in \R^R: \norm{\vx}_2 = 1}c_1 \sum_{\ell \in L}\sum_{r \in R} \vx_r^2 =c_1c_2.\]
\end{proof}

We also require the following expander existence result of \cite{gribinski2021existence}.

\begin{proposition}[Theorem 1.2, \cite{gribinski2021existence}]
\label{prop:expander_existence}
For any $(g, h, t) \in \N^3$, there exists a bipartite graph $G_{g,h,t} = (V = L \cup R,E)$ such that the following properties hold.
\begin{itemize}
    \item $|L| = tg$, $|R| = g$, $\deg(u) = h$ for all $u \in L$, and $\deg(v) = th$ for all $v \in R$.
    \item $\lam_2(\ma) \le \sqrt{h-1} + \sqrt{th-1}$, where $\ma$ is the adjacency matrix of $G_{g,h,t}$.
\end{itemize}
\end{proposition}

We now describe our $\kRCS$ covariance matrix construction, leveraging Proposition~\ref{prop:expander_existence}.

\begin{corollary}
\label{cor:kRCS_existence}
Let $k \ge 6$ be even, let $d = 3g$ for $g \ge k$, and let $R \subseteq [d]$ with $|R| = g$ be arbitrary, with $L \defeq [d] \setminus R$. There exists a $\kRCS$ $\msig \in \PSD^{d \times d}$ such that the following properties hold:
\begin{equation}\label{eq:krcs_graph_props}
\lam_1(\msig) = 2,\quad \lam_2(\msig) < 1.97,\quad \vv_1(\msig)=
    \sqrt{\frac{3}{4d}}
    \begin{pmatrix}
        \mathbf{1}_{L}\\
        \sqrt{2}\cdot \mathbf{1}_{R}
    \end{pmatrix}.
\end{equation}
\end{corollary}

\begin{proof}
Let $t = 2$ and $h = \frac{k - 2}{2}$ in Proposition~\ref{prop:expander_existence}, and let $\ma$ be the adjacency matrix
of the resulting $(h,2h)$-biregular  graph with bipartition $(L, R)$. We let
\[
\msig \defeq \id_d + \frac{\sqrt{2}}{k-2}\ma = \id_d + \frac 1 {\sqrt 2 h} \ma.
\]
The fact that $\msig$ is $\kRCS$ is immediate from $2h + 1 \le k$, and the remaining properties follow from combining Lemma~\ref{lem:biregular_top_eigen} and Proposition~\ref{prop:expander_existence}, as well as the numerical bound (for $k \ge 6$):
\[\lam_2(\msig) \le 1+\frac{\sqrt{2}}{k-2}\Big(\sqrt{h-1}+\sqrt{2h-1}\Big)
=
1+\frac{\sqrt{k-4}+\sqrt{2k-6}}{k-2} < 1.97.\]
\end{proof}




Combining Lemma~\ref{lem:gv_constant_weight} and Corollary~\ref{cor:kRCS_existence} in Proposition~\ref{prop:priv_lower} now yields our lower bound.

\begin{theorem}
\label{thm:priv_sparse_pca_pure_dp_lower_bound}
Let $\calA: (\R^d)^n \to \R^d$ be an $\eps$-DP algorithm that, under Model~\ref{model:general_krcs_cov} with $\sigma^2 = 1$ and $\gamma = \frac 1 {100}$, solves Problem~\ref{prob:pca} with $\beta = \frac 1 3$ and $\Delta = \frac 1 {400}$ for sufficiently large $d, k$. Then $n = \Omega(\frac d \eps)$.
\end{theorem}

\begin{proof}
Let $k \ge 6$ be even and $d = 3g$ for $g \ge k$. Let $C$ be the subset of $\{0, 1\}^d$ with size $m \defeq |C| = \exp(\Omega(d))$ guaranteed by Lemma~\ref{lem:gv_constant_weight}. For each $\vx \in C$, let $\msig_{\vx}$ be the $\kRCS$ matrix given by Corollary~\ref{cor:kRCS_existence} with $R \gets \supp(\vx)$. Using \eqref{eq:krcs_graph_props}, the Gaussian $P_{\vx} \defeq \Nor(\0_d, \half \msig_{\vx})$ satisfies Model~\ref{model:general_krcs_cov} with $\sigma^2 = 1$ and $\gamma = \frac 1 {100}$. The rest of the proof establishes that whenever $\calA$ returns $\hvv$ with
\begin{equation}\label{eq:alg_success}\sin^2 \angle(\hvv, \vv_1(\msig_{\vx})) \le \frac 1 {400},\end{equation}
there is a deterministic $D: \R^d \to C$ with $D(\hvv) = \vx$. Note that $D \circ \calA$ is also $\eps$-DP by postprocessing, so
Proposition~\ref{prop:priv_lower} with $\alpha = 1$, $\calA \gets D \circ \calA$, and $\calP = \{P_{\vx}\}_{\vx \in C}$, proves the theorem.

We now describe the decoding algorithm $D$. For all $\vx \in C$, denote by shorthand $\vv_{\vx} \defeq \vv_1(\msig_{\vx})$ (i.e., the vector in \eqref{eq:krcs_graph_props}).
The algorithm simply selects $\vy \in C$ that minimizes $\dsign(\vy, \hvv)$, where
\[\dsign(\vu, \vv) \defeq \min\Brace{\norm{\vu - \vv}_2,\norm{\vu + \vv}_2}.\]

Note that $\dsign$ satisfies the triangle inequality. Indeed, for any three unit vectors $\va, \vb,$ and $\vc$, let $\sigma, \sigma' \in \{\pm 1\}$ be signs for which $\dsign(\va, \vb) = \|\va-\sigma\vb\|_2$ and $\dsign(\vb, \vc) = \|\vb-\sigma' \vc\|_2$. Then,
\[
\dsign(\va, \vc) \le \|\va - \sigma \sigma' \vc\|_2 \le \|\va - \sigma \vb \|_2 + \|\sigma \vb - \sigma \sigma' \vc \|_2 = \dsign(\va, \vb) + \dsign(\vb, \vc).
\]

By Lemma~\ref{lem:sinesq_to_l2}, the true $\vx$ that indexes $P_{\vx}$ satisfies
\[\dsign(\vv_{\vx}, \hvv) \le \sqrt{\frac 1 {200}}.\]
On the other hand, for any $\vy \in C$ such that $\vy \neq \vx$, \eqref{eq:krcs_graph_props} implies that
\[\norm{\vv_{\vx} - \vv_{\vy}}_2 = \sqrt{\ham(\vx, \vy)} \Par{\sqrt{\frac 3 {2d}} - \sqrt{\frac 3 {4d}}} \ge \sqrt{\frac d 6} \Par{\sqrt{\frac 3 {2d}} - \sqrt{\frac 3 {4d}}} > \sqrt{\frac 1 {50}}. \]
Moreover, $\dsign(\vv_{\vx}, \vv_{\vy}) = \norms{\vv_{\vx} - \vv_{\vy}}$ because $\inprods{\vv_{\vx}}{\vv_{\vy}} > 0$. Finally,  $\dsign(\vv_{\vy}, \hvv) > \dsign(\vv_{\vx}, \hvv)$ for any $\vy \neq \vx$, because triangle inequality yields the following contradiction otherwise:
\[\sqrt{\frac 1 {50}} < \dsign(\vv_{\vx}, \vv_{\vy}) \le \dsign(\vv_{\vx}, \hvv) + \dsign(\hvv, \vv_{\vy}) \le 2\dsign(\vv_{\vx}, \hvv) \le 2\sqrt{\frac 1 {200}}.\]
Thus, $D$ correctly returns $\vx$ whenever \eqref{eq:alg_success} holds.
\end{proof}

\subsection{Approximate DP lower bound}\label{ssec:approx_pca}

In this section, we prove Theorem~\ref{thm:priv_sparse_pca_approx_dp_lower_bound}, our lower bound for PCA for $\kRCS$ covariance matrices with \emph{approximate DP}. We note that our hard instances follow a slightly different problem parameterization than used by Model~\ref{model:general_krcs_cov}, which makes it not fully compatible with our corresponding upper bounds in Section~\ref{sssec:priv_sparse_pca_upper_bound}; we provide additional commentary on this discrepancy in Remark~\ref{rem:ub_compatibility}.

Our lower bound arguments apply the following variant of the \emph{DP Assouad's lemma} from \cite{pmlr-v132-acharya21a}; for completeness, we defer a proof to Appendix~\ref{appendix:priv_sparse_pca_lower_bound}.

\begin{restatable}{proposition}{PrivateAssouad}\label{prop:private_assouad} Let $\mathcal{V} \subseteq \{\pm 1\}^{d}$ and associate each $\vs \in \mathcal{V}$ with a distribution $P_{\vs}$ over $\Omega^n$ for some sample space $\Omega$, and a parameter $\vtheta_{\vs} \in \Theta$. Let symmetric $\ell: \Theta \times \Theta$ satisfy
\begin{equation}\label{eq:loss_props}\begin{aligned}
     \ell\Par{\vtheta_{\vu}, \vtheta_{\vv}} \geq 2\tau \sum_{i \in [d]}\ind\bb{\vu_i \neq \vv_i}\text{ for all } (\vu, \vv) \in \calV^2, \\
     \ell\bb{\vtheta_{1}, \vtheta_{2}} \leq 2\bb{\ell\bb{\vtheta_{1}, \vtheta_{3}} + \ell\bb{\vtheta_{3}, \vtheta_{2}}} \text{ for all } (\vtheta_1, \vtheta_2, \vtheta_3) \in \Theta^3.
\end{aligned}\end{equation}
For each $i \in [d]$, let $\mathcal{V}_{+i} := \{\vs \in \mathcal{V}: \vs_{i} = +1\}$, $\mathcal{V}_{-i} := \{\vs \in \mathcal{V}: \vs_{i} = -1\}$, and  define
\bas{
     P_{+i} := \frac{1}{|\mathcal{V}_{+i}|}\sum_{\vs \in \mathcal{V}_{+i}}P_{\vs}, \quad P_{-i} := \frac{1}{|\mathcal{V}_{-i}|}\sum_{\vs \in \mathcal{V}_{-i}}P_{\vs}.
}
If for all $i \in [d]$ with $\min\{|\mathcal{V}_{-i}|, |\mathcal{V}_{+i}|\} > 0$, there is a coupling $\Gamma$ of $\calX \defeq \{\vx_j\}_{j \in [n]} \sim P_{+i}$ and $\calY \defeq \{\vy_j\}_{j \in [n]} \sim P_{-i}$ with
\begin{equation}\label{eq:coupling_hamming}\E_{(\calX, \calY) \sim \Gamma}\Brack{\sum_{j \in [n]} \ind(\vx_j \neq \vy_j)} \le D ,\end{equation}
then defining $R(\calV, \ell, \eps, \delta) \defeq \min_{\alg \text{ is } (\eps, \delta)\text{-DP}} \max_{\vs \in \calV} \E_{\calX \sim P_{\vs}^{\otimes n}}[\ell(\calA(\calX), \vtheta_{\vs})]$,
\[R(\calV, \ell, \eps, \delta) \ge \frac{\tau}{2}\cdot\frac{\sum_{i \in [d]}\min\{|\mathcal{V}_{-i}|, |\mathcal{V}_{+i}|\}}{|\mathcal{V}|}\cdot\bb{0.9e^{-10\eps D} - 10D\delta}. \]
\end{restatable}

We start by giving a graph-based construction of a family of $\kRCS$ covariance matrices.

\begin{lemma} \label{lem:ramanujan_cov_family}
Let $k, d$ be sufficiently large, and for each sign vector $\vs \in \{\pm 1\}^d$, define the diagonal sign matrix $\md_{\vs} \defeq \diags{\vs}$. There exists a matrix $\ma \in \{0, 1\}^{d \times d}$ such that if we let
\[\msig_{\vs} \defeq \id_d + \frac 1 {k - 1} \md_{\vs} \ma \md_{\vs},\]
the following properties hold: $\msig_{\vs} \in \PSD^{d \times d}$ is $\kRCS$,
\[\lam_1(\msig_{\vs}) = 2,\quad \lam_2(\msig_{\vs}) \le \frac 3 2,\quad \vv_{\vs} \defeq \vv_1(\msig_{\vs}) = \frac 1 {\sqrt d} \vs.\]
\end{lemma}

\begin{proof} Taking $g \gets \frac d 2$, $t=1$, and $h \gets k - 1$ in Proposition~\ref{prop:expander_existence}, let $\ma$ be the adjacency matrix
of the resulting $(h,h)$-biregular bipartite graph, $G$. Multiplication by $\md_\vs$ preserves the zero pattern, hence each row of
$\md_\vs\ma\md_\vs$ has $h$ nonzeros, so $\msig_{\vs}$ is $\kRCS$. Since $\md_\vs$ is orthogonal, $\md_\vs\ma\md_\vs$ is similar to $\ma$, hence they share eigenvalues.
Because $G$ is $h$-regular and bipartite, $\lam_1(\ma) = \normsop{\ma} =h$ (Lemma~\ref{lem:biregular_top_eigen}).
Since the eigenvalues of $\msig_\vs$ are $\{1 + \frac 1 h \lambda_j(\ma)\}_{j \in [d]}$, $\msig_{\vs} \in \PSD^{d \times d}$, and
\[
\lam_1(\msig_\vs) = 1 + \frac{h}{h} = 2,\quad \lam_2(\msig_{\vs}) = 1 + \frac{2\sqrt{h - 1}}{h} \le \frac 3 2,
\]
and the statement about $\vv_{\vs}$ follows from Lemma~\ref{lem:biregular_top_eigen} after conjugation by $\md_{\vs}$.
\end{proof}


We are now ready to give our family of distributions for use in Proposition~\ref{prop:private_assouad}.

\begin{lemma}\label{lem:edge_spike}
Let $G = (V, E)$ be $r$-regular and let $\vs \in \{\pm 1 \}^d$. Let $\dirE$ be made by taking each undirected $e \in E$ and adding both corresponding directed edges to $\dirE$, so that $|\dirE| = rd$. Define $\calD_{\vs, G}$ to be the following distribution. Independently draw $g \simunif \{\pm 1\}$ and $(u, v) \simunif \dirE$, and output
\[\vx \defeq \sqrt{\frac d 2} g \cdot  \md_{\vs}\Par{\ve_u + \ve_v}.\]
Then $\calD_{\vs, G}$ is sub-Gaussian (Definition~\ref{def:subgaussianity}) with $\sigma^2 = d$, and following Lemma~\ref{lem:ramanujan_cov_family},
\[\E_{\calD_{\vs, G}}[\vx] = \0_d,\quad \E_{\calD_{\vs, G}}[\vx\vx^\top] = \msig_{\vs}.\]
\end{lemma}
\begin{proof}
We first verify the moment calculations: $\E_{\calD_{\vs, G}}[\vx] = \0_d$ follows since $\E[g] = 0$, and
\[\E_{\calD_{\vs, G}}\Brack{\vx\vx^\top} = \E_{(u, v) \simunif \dirE}\Brack{\E\Brack{\vx\vx^\top \mid (u, v)}} = \frac d 2 \md_{\vs}\E_{(u, v) \simunif \dirE}\Brack{\Par{\ve_u + \ve_v}\Par{\ve_u + \ve_v}^\top } \md_{\vs}.\]
Next, it is a straightforward calculation that
\[\E_{(u, v) \simunif \dirE}\Brack{\Par{\ve_u + \ve_v}\Par{\ve_u + \ve_v}^\top } = \frac 2 d \id_d + \frac 2 {|\dirE|}\sum_{(u, v) \in \dirE} \ve_u\ve_v^\top= \frac 2 d \id_d + \frac 2 {rd} \ma. \]
The covariance follows by combining the above two displays with the definition of $\msig_{\vs}$ in Lemma~\ref{lem:ramanujan_cov_family}. Finally, the sub-Gaussianity claim follows because for any $\vu \in \R^d$, and any $(u, v) \in \dirE$,
\begin{align*}
\vu^\top \md_{\vs} (\ve_u + \ve_v) \le \sqrt 2 \norm{\vu}_2 \implies
\E\Brack{\exp\Par{\vu^\top \vx} \mid (u, v)} \le \exp\Par{\frac{d}{4} \cdot 2\norm{\vu}_2^2} \le \exp\Par{\frac d 2 \norm{\vu}_2^2},
\end{align*}
by using $1$-sub-Gaussianity of $g$. Taking expectation over $(u, v) \simunif \dirE$ completes the proof.
\end{proof}

At this point, we have assembled all the pieces required for our lower bound.

\begin{theorem}\label{thm:priv_sparse_pca_approx_dp_lower_bound}
Let $\calA: (\R^d)^n \to \R^d$ be an $(\eps, \delta)$-DP algorithm that, under Model~\ref{model:general_krcs_cov} with $\sigma^2 = d$ and $\gamma = \frac 1 4$, solves Problem~\ref{prob:pca} with $\beta = \Delta \le \frac 1 {25}$ for sufficiently large $d, k$, and $\delta \le \frac \eps {100}$. Then $n = \Omega(\frac d \eps)$.
\end{theorem}
\begin{proof}
Let $m \defeq \frac d 2$ and let $G$ be the $(k - 1, k - 1)$-biregular bipartite graph $G$ from Lemma~\ref{lem:ramanujan_cov_family}. For each $\vx \in \{\pm 1\}^m$, let $\vs(\vx) \in \{\pm 1 \}^d$ concatenate $\1_{m}$. Let $\calV \defeq \{\vs \in \{\pm 1 \}^d: \vs = \vs(\vx) \text{ for some } \vx \in \{\pm 1 \}^m\}$. Then for all $\vs, \vs' \in \calV$ with $\ham(\vs, \vs') = t$,
\[\inprod{\vs}{\vs'} = 1 - \frac{2t}{d} \ge 0 \implies \sin^2 \angle(\vv_{\vs}, \vv_{\vs'}) \ge \frac{2t}{d}.\]
Thus, taking $\Theta$ to be the unit sphere and $\vtheta \equiv \vv$, the function $\ell(\vv_{\vs}, \vv_{\vs'}) \defeq \sin^2 \angle(\vv_{\vs}, \vv_{\vs'})$ satisfies \eqref{eq:loss_props} with $\tau = \frac 1 d$. The other property in \eqref{eq:loss_props} follows from Lemma~\ref{lem:sinesq_to_l2}.

Next, to each $\vs \in \calV$ we associate a distribution $P_{\vs}$ given by $n$ independent draws from $D_{\vs, G}$ (as defined in Lemma~\ref{lem:edge_spike}). We claim that for all pairs of $\vs \in \calV_{+i}$, $\vs' \in \calV_{-i}$ differing only in the $i^{\text{th}}$ coordinate, we have $\TV(D_{\vs, G}, D_{\vs', G}) \le \frac 2 d$. To see this, upon coupling the random sign $g \in \{\pm 1\}$ and random edge $(u, v) \in \dirE$ used in constructing $D_{\vs, G}, D_{\vs', G}$, the samples differ iff $i \in \{u, v\}$. This occurs with probability $\le \frac 2 d$. Thus, letting $\Gamma$ in \eqref{eq:coupling_hamming} be the product coupling on coordinates, where each coordinate coupling is the TV coupling, shows we may take $D = \frac{2n}{d}$ in \eqref{eq:coupling_hamming}.

Now, suppose that for some $n = o(\frac d \eps)$ there exists an $(\eps, \delta)$-DP algorithm $\calA$ that solves Problem~\ref{prob:pca} with the stated parameters. Applying Proposition~\ref{prop:private_assouad} shows that
\[R(\calV, \ell, \eps, \delta) \ge \frac \tau 2 \cdot \frac{d}{4} \cdot \Par{0.9 e^{-\frac{20\eps n}{d}} - \frac{20n\delta}{d}} = \frac 1 8 \Par{0.9 e^{-\frac{20\eps n}{d}} - \frac{20n\delta}{d}} > \frac 1 {10},\]
for small enough $n$. On the other hand, since $\calA$ achieves $\beta = \Delta \le \frac 1 {20}$ in Problem~\ref{prob:pca}, it certifies that $R(\calV, \ell, \eps, \delta) \le \frac 1 {10}$, because the worst-case $\sin^2$ error is $\le 1$. This is a contradiction.
\end{proof}

\begin{remark}\label{rem:ub_compatibility}
For comparison, the construction in Theorem~\ref{thm:priv_sparse_pca_approx_dp_lower_bound} has two natural non-private
benchmarks. If one applies the generic Model~\ref{model:general_krcs_cov} analysis based only on
$\sigma$-sub-Gaussianity and $k$-RCS structure, thresholding gives
$\normsop{\msig-\hmsig}
    = O(k\sigma^2\sqrt{\log d/n})$.
Since here $\sigma^2=d$, $\lam_1(\Sigma_s)=\Theta(1)$, and the eigengap is
constant, this generic route requires
$n = \Omega(d^2k^2\log d)$ samples for constant PCA error. This reflects the
spikiness of the construction rather than an intrinsic non-private difficulty.

For the specific edge-spike family in Theorem~\ref{thm:priv_sparse_pca_approx_dp_lower_bound}, there is also a simple
distribution-specific non-private interpretation. Each sample reveals a sampled edge
$(U,V)$ of the underlying graph together with the parity $s_Us_V$, since
$\operatorname{sign}(x_Ux_V)=s_Us_V$. Thus recovering the leading eigenvector
$v_s=s/\sqrt d$ reduces non-privately to recovering the vertex signs from sampled
edge parities, up to a global sign. This gives the elementary bounds $\Omega(d)= n_{\mathrm{nonpriv,family}}= O(kd\log d)$
where the upper bound follows by coupon-collecting all $O(kd)$ edges of the base
graph. Thus $\widetilde O(kd)$ samples suffice
non-privately.

Theorem~\ref{thm:priv_sparse_pca_approx_dp_lower_bound} shows that under approximate DP, any
algorithm with constant expected $\sin^2$ error requires
$n=\Omega(d/\varepsilon)$. Hence the theorem should be viewed as a partial
approximate-DP lower bound showing an ambient-dimensional privacy barrier for a
spiky $k$-RCS family.
\end{remark}

\bibliographystyle{alpha}
\bibliography{refs}

\section*{Acknowledgments}
The authors thank Eric Price for many beneficial conversations that led to the construction of the pure-DP lower bound. SK gratefully acknowledges funding support from the Amazon AI PhD Fellowship. PS gratefully acknowledges NSF grants 2217069 and CCF-2505865.

\newpage
\appendix

\section{Utility Results}\label{appendix:utility_results}



In this section, we prove several utility lemmas that are used throughout the paper.

\restatethreshold*
\begin{proof} 
We start with the first claim. 
Define matrices $\mx \in \{0, 1\}^{d \times d}$, $\my \in \R^{d \times d}_{\ge 0}$, where
\[\mx_{ij} = \ind\Par{\ma_{ij} \neq 0},\quad \my_{ij} \defeq \Abs{\Brack{\calT_\rho(\mb) - \ma}_{ij}},\text{ for all } (i, j) \in [d] \times [d].\]
Because $\rho \ge \tau$, each entry of $\my$ is nonzero only if the same entry of $\mx$ is nonzero, and
\[0 \le \my_{ij} \le \Par{\rho + \tau} \mx_{ij} \le 2\rho \mx_{ij}, \text{ for all } (i, j) \in [d] \times [d]. \]
Thus, defining $\mz \defeq \calT_\rho(\mb) - \ma$ (so that $\my = |\mz|$ for $|\cdot|$ applied entrywise),
\begin{equation}\label{eq:entrywise_abs}
\begin{aligned}
\normop{\mz} &= \max_{\substack{\vx, \vy \in \R^d \\ \norm{\vx}_2 = \norm{\vy}_2 = 1}}\sum_{(i, j) \in [d] \times [d]} \vx_i \vy_j \mz_{ij} \le \max_{\substack{\vx, \vy \in \R^d \\ \norm{\vx}_2 = \norm{\vy}_2 = 1}}\sum_{(i, j) \in [d] \times [d]} |\vx_i| |\vy_j| \my_{ij} \\
&\le  2\rho\max_{\substack{\vx, \vy \in \R^d \\ \norm{\vx}_2 = \norm{\vy}_2 = 1}}\sum_{(i, j) \in [d] \times [d]} |\vx_i| |\vy_j| \mx_{ij} = 2\rho\normop{\mx}, 
\end{aligned}
\end{equation}
where we used the Perron-Frobenius theorem in the last line to conclude that the maximizing $\vx, \vy$ are entrywise nonnegative. 
It remains to bound $\normsop{\mx}$ when $\ma$ is $\kRCS$:
\[\normop{\mx} = \max_{\vv \in \R^d: \norm{\vv}_2 \neq 0} \frac{\norm{\mx \vv}_2}{\norm{\vv}_2} \le \sqrt{\Par{\max_{\vv \in \R^d: \norm{\vv}_\infty \neq 0} \frac{\norm{\mx \vv}_\infty}{\norm{\vv}_\infty}} \cdot \Par{\max_{\vv \in \R^d: \norm{\vv}_1 \neq 0} \frac{\norm{\mx \vv}_1}{\norm{\vv}_1}}} \le k,\]
where the first inequality is Problem 5.6.P21 in \cite{HornJ12}, and the second uses the $\kRCS$ assumption. 

We next prove the second claim. Let $S \defeq \supp(\Top_s(\mb))$ and let $T \defeq \supp(\ma)$. We first have
\[\Abs{\Brack{\Top_s(\mb) - \ma}_{ij}} = \Abs{\Brack{\mb - \ma}_{ij}} \le \rho \text{ for all } (i, j) \in S.\]
Conversely, every $(i, j) \not\in S$ has $|\mb_{ij}| \le \rho$. This is because there are at most $s$ entries of $\mb$ with magnitudes $>\rho$ (since $\nnz(\ma) \le s$ and $\normsinf{\mb - \ma} \le \rho$), so they are all kept by $S$. Thus, 
\[\Abs{\Brack{\Top_s(\mb) - \ma}_{ij}} = \Abs{\ma_{ij}} \le \Abs{\Brack{\mb - \ma}_{ij}} + \Abs{\mb_{ij}} \le  \tau + \rho \le 2\rho\text{ for all } (i, j) \not\in S. \]
We have shown that all of the $\le 2s$ nonzero entries of the matrix $\Top_s(\mb) - \ma$ have magnitude $\le 2\rho$, and the conclusion follows from $\normsop{\Top_s(\mb) - \ma} \le \normsf{\Top_s(\mb) - \ma}$.
\end{proof}

\begin{lemma}
\label{lem:sinesq_to_l2}
For any unit vectors $\va, \vb \in \R^d$, let $\ell(\va,\vb)\defeq \sin^2 \angle(\va, \vb) = 1-\inprods{\va}{\vb}^2$ and
$\dsign(\va, \vb) \defeq \min\Brace{\norm{\va - \vb}_2,\norm{\va + \vb}_2}.$ Then,
\[
\dsign(\va, \vb)^2 \le 2\ell(\va,\vb).
\]
Moreover, for all unit vectors $\vx,\vy,\vz\in\R^d$,
\[
\ell(\vx,\vy)
\leq
2\ell(\vx,\vz)+2\ell(\vz,\vy).
\]

\end{lemma}

\begin{proof}
Let $\rho \defeq \inprods{\vx}{\vy}\in[-1,1]$ and choose $s$ so that $\inprods{\vx}{s\vy}=|\rho|\ge 0$. Note that $\|\vx-s\vy\|_2^2 = 2(1-|\rho|) \le 2(1+|\rho|) = \|\vx + s\vy\|_2^2$. 
So, $\dsign(\vx, \vy) = \|\vx - s\vy\|_2$.
Since $1 \le 1+|\rho| \le 2,$
\[
\|\vx-s\vy\|_2^2 = 2(1 - |\rho|) \le 2(1+|\rho|)(1-|\rho|) = 2(1-\rho^2) = 2\ell(\vx, \vy).
\]
Next, we define the matrices
\[
\mpp_{\vx}\defeq\vx\vx^\top,
\qquad
\mpp_{\vy}\defeq\vy\vy^\top,
\qquad
\mpp_{\vz}\defeq\vz\vz^\top.
\]
Note that for any unit vectors $\va$ and $\vb$,
\[
\normf{\mpp_{\va}-\mpp_{\vb}}^2
=
2-2\inprods{\va}{\vb}^2
=
2\ell(\va,\vb).
\]
By the triangle inequality on the Frobenius norm,
\begin{align*}
2\ell(\vx,\vy) =
\normf{\mpp_{\vx}-\mpp_{\vy}}^2 &\leq
\Pars{
\normf{\mpp_{\vx}-\mpp_{\vz}}
+
\normf{\mpp_{\vz}-\mpp_{\vy}}
}^2 \\
&\leq
2\normf{\mpp_{\vx}-\mpp_{\vz}}^2
+
2\normf{\mpp_{\vz}-\mpp_{\vy}}^2 =
4\ell(\vx,\vz)+4\ell(\vz,\vy).
\end{align*}
Dividing by two proves the second inequality.
\end{proof}

\begin{lemma}\label{lem:rayleigh_gap}
Let \(\ma\in\PSD^{d\times d}\) have eigenvalues
\(\lambda_1(\ma)>\lambda_2(\ma)\), and let \(\vv \defeq\vv_1(\ma)\).
Then for every unit vector \(\vu\in\R^d\),
\[
\sin^2\angle(\vu, \vv) \le \frac{\lambda_1(\ma)-\vu^\top\ma\vu}{\lambda_1(\ma)-\lambda_2(\ma)}.
\]
\end{lemma}
\begin{proof}
Write \(\vu=c\vv+s\vw\), where \(\vw\perp\vv\) is a unit vector. Then
since \(\vw\perp\vv\),
\[
\vu^\top\ma\vu
=
c^2\lambda_1(\ma)+s^2 \vw^\top\ma\vw
\le
c^2\lambda_1(\ma)+s^2\lambda_2(\ma).
\]
The claim follows from \(s=\sin\angle(\vu,\vv)\), and rearranging
\[
\lambda_1(\ma)-\vu^\top\ma\vu
\ge
(\lambda_1(\ma)-\lambda_2(\ma))\sin^2\angle(\vu, \vv).
\]
\end{proof}

\begin{lemma}\label{lem:dp_expectation_comparison} Let $P$ and $Q$ be probability distributions on a measurable space such that, for every measurable set $A$, \[ P(A)\leq e^\eps Q(A)+\delta, \qquad Q(A)\leq e^\eps P(A)+\delta . \] Let $\eps\leq 1$, and let $h$ be a measurable function that is square-integrable under both $P$ and $Q$. Then \[ \Abs{\E_P h-\E_Q h} \leq 2\eps\E_Q\Abs h + 2\sqrt{\delta}\sqrt{\E_P[h^2]+\E_Q[h^2]}. \] \end{lemma} \begin{proof} Let $\mu\defeq P+Q$, and let $p\defeq\dd P/\dd\mu$ and $q\defeq\dd Q/\dd\mu$. The two inequalities imply \begin{equation}\label{eq:hockey_stick_masses} \int (p-e^\eps q)_+\,\dd\mu\leq\delta, \qquad \int (q-e^\eps p)_+\,\dd\mu\leq\delta . \end{equation} Indeed, the first integral equals $\sup_A\{P(A)-e^\eps Q(A)\}$, with the supremum attained at $A=\{p>e^\eps q\}$; the second identity is symmetric. For any scalar $c$, denote $(c)_+ \defeq \max\{c, 0\}$, and define $r_P\defeq(p-e^\eps q)_+$ and $r_Q\defeq(q-e^\eps p)_+$. We claim that \begin{equation}\label{eq:density_difference_bound} \Abs{p-q}\leq(e^\eps-1)q+r_P+r_Q . \end{equation} If $p\geq q$, then $p-q\leq(e^\eps-1)q+r_P$. If $q>p$, then $q-p\leq(e^\eps-1)p+r_Q\leq(e^\eps-1)q+r_Q$. This proves \eqref{eq:density_difference_bound}. Using \eqref{eq:density_difference_bound}, \begin{align} \Abs{\E_P h-\E_Q h} &= \Abs{\int h(p-q)\,\dd\mu} \leq (e^\eps-1)\E_Q\Abs h +\int\Abs h\,r_P\,\dd\mu +\int\Abs h\,r_Q\,\dd\mu . \label{eq:expectation_density_split} \end{align} Since $r_P\leq p$, $r_Q\leq q$, and both integrate to at most $\delta$ by \eqref{eq:hockey_stick_masses}, Cauchy--Schwarz gives \[ \int\Abs h\,r_P\,\dd\mu \leq \sqrt{\delta\,\E_P[h^2]}, \qquad \int\Abs h\,r_Q\,\dd\mu \leq \sqrt{\delta\,\E_Q[h^2]}. \] Using $\sqrt a+\sqrt b\leq\sqrt{2(a+b)}\leq2\sqrt{a+b}$ and $e^\eps-1\leq2\eps$ for $\eps\leq1$ proves the claim. \end{proof}

\section{Deferred Proofs from Section~\ref{ssec:priv_sparse_cov_est}}\label{appendix:priv_sparse_cov_est}

\invwishprop* \begin{proof}
The first two claims are immediate from Construction~\ref{const:fingerprinting_prior}.
To prove Item~\ref{item:cov_prop_3}, we start from Lemma 3.7, \cite{narayanan2024better}, which implies that for $\mm \sim \InvWishart(\id_k, 2k)$ and all $x > 0$,
\[\Prob\Brack{\lam_1(\mm) \ge \frac x {2k}} \le \Par{\frac{e^2}{x}}^a.\]
Moreover, the law of $(k - 1) \mm$ for $\mm \sim \InvWishart(\id_k, 2k)$ is $\InvWishart((k - 1) \id_k, 2k)$, i.e., the law of each $\msig_b$. Thus, setting $\frac x {2k} \gets \frac{t}{k-1}$ and using a union bound over all $B = \frac d k$ blocks,
\[\Pr\Brack{\max_{b \in [B]} \normop{\msig_b} > t} = \Pr\Brack{\normop{\msig} > t} \le \frac d k\Par{\frac{3e^2}{t}}^a.\]
To simplify notation, let $A \defeq 3e^2$ and $t_0 \defeq A(\frac d k)^{1/a}$ henceforth. 
Using $\E[Z^q] = \int_0^\infty q t^{q-1}\Prob[Z\ge t] \,\dd t$ for any nonnegative random variable $Z$, we obtain
\begin{align*}
\E\Brack{\normop{\msig}^q}
&\le
\int_0^{t_0} q t^{q-1}\,\dd t
+
\int_{t_0}^{\infty} q t^{q-1}\,\frac{d}{k}\Big(\frac{A}{t}\Big)^{a}\,\dd t \\
&=
t_0^{q}
+
\frac{qd A^a}{k}\int_{t_0}^{\infty} t^{q-1-a}\, \dd t \\
&= t_0^q + \frac{qdA^a}{k} \cdot \frac{t_0^{q - a}}{a-q}  = t_0^q \cdot \frac{a}{a-q} \\
&\le 2t_0^q \le 2\exp\Par{4q} \Par{\frac d k}^{\frac q a} \le 2\exp\Par{6q}.
\end{align*}
The first line above split the integral and used that probabilities are always $\le 1$, the next two lines evaluated the integrals using our assumption $q < a$, and the last line used $A \le e^4$ and $(\frac d k)^{1/a} \le e^2$ by our lower bound on $k$ in Construction~\ref{const:fingerprinting_prior}.

For the first part of Item~\ref{item:cov_prop_4}, Lemma 3.7 from \cite{narayanan2024better} implies that $\E[\lam_k(\msig_b)^2] \ge c$ for a universal constant $c$, and we have obtained a matching upper bound (up to a constant factor) by using Item~\ref{item:cov_prop_3} with $q = 2$. Therefore, by Lemma 3.1 from \cite{narayanan2024better}, for universal constants $0 < g_1 < g_2$,
\bas{
    g_{1}\frac{k^{2}}{n} \leq \E\Brack{\normf{\hmsig_b-\msig_b}^2} \leq g_{2}\frac{k^{2}}{n} .
}
The second part of Item~\ref{item:cov_prop_4} follows by summing the above display over all of the $B = \frac d k$ blocks.
\end{proof}

\bayesianfingerprinting* 
\begin{proof} All expectations below include the internal randomness of $\calA$, which is independent of the prior and sample distributions. We first show that the replacement statistic $Z_i'$ is mean-zero. Write $Y_i'\defeq \calA(\calX^{\sim i})$. Conditional on $(\theta,\calX^{\sim i},Y_i')$, $X_i$ remains distributed as $P_\theta$, since $\calX^{\sim i} \independent X_i$, and the internal randomness producing $Y_i'$ from $\calX^{\sim i}$ is independent of all samples. Therefore, \begin{align} \E[Z_i'\mid\theta,\calX^{\sim i},Y_i'] &= \inprod{Y_i'-g(\theta)} {\E[\psi(\theta,X_i)\mid\theta,\calX^{\sim i},Y_i']} \notag\\ &= \inprod{Y_i'-g(\theta)} {\E[\psi(\theta,X_i)\mid\theta]} =0. \label{eq:replacement_score_centered} \end{align} Taking expectations gives $\E[Z_i']=0$. Next, for all $i \in [n]$, let \[ W_i\defeq(\theta,X_1,\ldots,X_n,X_i'). \] Conditional on $W_i$, the datasets $\calX$ and $\calX^{\sim i}$ are fixed and neighboring. Let $P_{W_i}$ and $Q_{W_i}$ denote the respective output distributions of $\calA(\calX)$ and $\calA(\calX^{\sim i})$. By Lemma~\ref{lem:dp_expectation_comparison} applied with 
\[P \gets P_{W_i},\quad Q \gets Q_{W_i},\quad h \gets h_{W_i}\text{ where } h_{W_i}(y)\defeq\inprod{y-g(\theta)}{\psi(\theta,X_i)},\] 
we have for $Z_i \defeq h_{W_i}(\calA(\calX))$, $Z'_i \defeq h_{W_i}(\calA(\calX^{\sim i}))$, that
\[
\Abs{\E[Z_i-Z_i'\mid W_i]}
\leq
2\eps\E[\Abs{Z_i'}\mid W_i]
+
2\sqrt{\delta}
\sqrt{\E[Z_i^2\mid W_i]+\E[(Z_i')^2\mid W_i]}.
\]
Taking expectations over $W_i$, and applying $\E[Z'_i] = 0$ gives
\begin{equation}\label{eq:single_score_dp_bound}
\begin{aligned}
\Abs{\E Z_i} = \Abs{\E[Z_i-Z_i']}
&\leq
\E\Abs{\E[Z_i-Z_i'\mid W_i]} \\
&\leq
2\eps\E\Abs{Z_i'}
+
2\sqrt{\delta}\,
\E\sqrt{\E[Z_i^2\mid W_i]+\E[(Z_i')^2\mid W_i]} \\
&\leq
2\eps\E\Abs{Z_i'}
+
2\sqrt{\delta}\sqrt{\E[Z_i^2]+\E[(Z_i')^2]} \leq
U. 
\end{aligned}
\end{equation}
Here, the first line used convexity of $|\cdot|$, the second used our previous upper bound, and the third applied concavity of $\sqrt{\cdot}$ as well as the definition of $U$.
Finally, the claim follows by applying the assumed lower bound and rearranging:
\[
L
\leq
\E\Brack{\sum_{i \in [n]} Z_i}
\leq
\sum_{i \in [n]}\Abs{\E Z_i}
\leq
nU.
\]
\end{proof}

\restatepostconc*
\begin{proof} The prior and observation distributions both factor across blocks $b \in [B]$, so
$\E[\msig_b\mid\calX]=\E[\msig_b\mid\calX_b]$, where
$\calX_b\defeq \{\vx_i^b\}_{i \in [n]}$. Fact~\ref{fact:inv_wishart_facts} then gives
\[
\E[\msig_b\mid\calX_b]-\hmsig_b
=
(1-w)\Pars{\E[\msig_b]-\hmsig_b},
\quad
1-w=\frac{k-1}{n+k-1}.
\]
We then have
    \bas{\E_{\calX_b}\bbb{\normf{\E\bbb{\msig_{b} \mid \calX} - \hmsig_{b}}^{2}} &= \E_{\calX_{b}}\bbb{\normf{\E\bbb{\msig_{b} \mid \calX_{b}} - \hmsig_{b}}^{2}} \\
    &= (1-w)^{2}\E_{\calX_{b}}\bbb{\normf{\hmsig_{b} - \E\bbb{\msig_{b}}}^{2}} \\
    &\leq 3(1-w)^{2}\E_{\calX_{b}}\bbb{\normf{\msig_{b}}^{2} + \normf{\hmsig_{b} - \msig_{b}}^{2} + \normf{\E\bbb{\msig_b}}^{2}} \\
    &\leq \bb{\frac{k-1}{n+k-1}}^{2}\cdot O\bb{k + \frac{k^{2}}{n} + k}  = O\bb{\frac{k^{3}}{n^2} + \frac{k^{4}}{n^{3}}}.
    }
The third line above used the triangle inequality with $(a + b + c)^2 \le 3(a^2 + b^2 + c^2)$, and the fourth line used Fact~\ref{fact:inv_wishart_facts}, and Items~\ref{item:cov_prop_4} and~\ref{item:cov_prop_3} of Lemma~\ref{lem:wishart_dist_prop}, to bound the three resulting terms (in that order).
Summing over the $\frac d k$ independent blocks proves the claim.
\end{proof}
\section{Private Operator Norm Estimation under Model~\ref{model:general_krcs_cov}}\label{appendix:priv_sparse_pca_upper_bound}

\begin{algorithm}[H]
\caption{\label{alg:friendly_DP_projector_opnorm}
$\kRCSOPNORM(\calD,\gamma,k,\eps,\delta,\tau,m)$}
\textbf{Input:} Dataset $\calD=\{\vx_i\in\R^d\}_{i\in[n]}$, gap 
parameter $\gamma>0$, sparsity parameter $k \in [d]$, privacy parameters $(\eps,\delta)$,
threshold $\tau$, number of batches $m \in \N$
\begin{algorithmic}[1]
\State $\hmsig_i \gets \frac 1 b \sum_{j = (i - 1)b + 1}^{ib} \vx_j\vx_j^\top$ for all $i \in [m]$, where $b \defeq \frac n m$
\State $\calC\gets\{i\in[m]:
\thresh_\tau(\hmsig_i)\text{ is }\kRCS,\
\lam_1(\thresh_\tau(\hmsig_i))>0,\
\frac{\lam_2(\thresh_\tau(\hmsig_i))}
{\lam_1(\thresh_\tau(\hmsig_i))}\le 1 - \frac\gamma2\}$
\State $L\sim\BLap\Pars{\frac 3 \eps, \frac 3 \eps \log(\frac{12}{\delta})}$
\If{$|\calC|+L- \frac 3 \eps \log(\frac{12}{\delta})<0.8m$}
  \State \Return $\perp$
\EndIf
\For{$i\in[m]$}
  \State $f_i\gets\sum_{j\in[m]}
  \ind\Pars{\normsinf{\hmsig_j-\hmsig_i}\leq2\tau}$.
  \State $p_i\gets\min\{\max\{\frac{f_i-m/2}{m/6},0\},1\}$.
\EndFor
\State $Z\gets\sum_{i\in[m]}p_i$
\State $\xi\sim\BLap\Pars{\frac{21}{\eps}, \frac{21}{\eps}\log(\frac{12}{\delta})}$
\If{$Z+\xi-\frac{21}{\eps}\log(\frac{12}{\delta})<0.8m$}
  \State \Return $\perp$
\EndIf
\State $\bmsig\gets \frac 1 Z\sum_{i\in[m]}p_i\hmsig_i$\label{line:agg_sig_opnorm}
\State
$\sigma_{\mathrm{priv}}\gets 20k\tau \cdot 
\frac 6 \eps\sqrt{\log(\frac 6 \delta)}$
\State $g\sim\calN(0,\sigma_{\mathrm{priv}}^2)$
\State \Return $\normsop{\thresh_{5\tau}(\bmsig)}+g$
\end{algorithmic}
\end{algorithm}

In this section, we provide an algorithm for privately estimating the operator norm of a $\kRCS$ covariance matrix under Model~\ref{model:general_krcs_cov}, by proving Proposition~\ref{thm:priv_sparse_cov_operator_norm_estimation}, which is used in Theorem~\ref{thm:priv_sparse_pca_upper_bound_gamma} and may be of independent interest. Our Algorithm~\ref{alg:friendly_DP_projector_opnorm} is patterned off of
Algorithm~\ref{alg:friendly_DP_projector}, with the same choice of $\tau$ as in \eqref{eq:tau_choice_pca}, deviating in the definition of the certified set and the final output.

\restateopnormest*

We split the proof of Proposition~\ref{thm:priv_sparse_cov_operator_norm_estimation} into two parts: Lemma~\ref{lem:private_opnorm_privacy} (which handles the privacy), and Lemma~\ref{lem:private_opnorm_utility} (which handles the utility). As much of the proofs are identical to the components used to prove Theorem~\ref{thm:priv_sparse_pca_upper_bound_gamma}, we primarily highlight the differences.

\begin{lemma}\label{lem:private_opnorm_privacy}
Algorithm~\ref{alg:friendly_DP_projector_opnorm} is $(\eps,\delta)$-DP.
\end{lemma}
\begin{proof} 
Algorithm~\ref{alg:friendly_DP_projector_opnorm} is identical to Algorithm~\ref{alg:friendly_DP_projector} through Line~\ref{line:agg_sig_opnorm}, except the definition of the set $\calC$ excludes the $\hlam$-dependent part of goodness (Definition~\ref{def:good}). Because the proof of Lemma~\ref{lem:sensitivity_barSigma_gamma} did not use the $\hlam$-dependent part of Definition~\ref{def:good}, its conclusion still holds, and therefore conditioned on Line~\ref{line:agg_sig_opnorm} being reached, the scalar statistic $\normsop{\calT_{5\tau}(\bmsig)}$ is $20k\tau$-sensitive. The conclusion follows from combining the Gaussian mechanism (Fact~\ref{fact:gaussian_mech}) with the privacy loss of reaching Line~\ref{line:agg_sig_opnorm} (Lemma~\ref{lem:dp_noisy_cert_count_gamma}).
\end{proof}

\begin{lemma}\label{lem:private_opnorm_utility}
In the setting of Proposition~\ref{thm:priv_sparse_cov_operator_norm_estimation} with $n$ taken as stated, $\tau$ set as in \eqref{eq:tau_choice_pca}, and $m = \Theta(\frac 1 \eps \log \frac 1 {\delta\beta})$ for an appropriate constant, Algorithm~\ref{alg:friendly_DP_projector_opnorm} returns $\hlam$ satisfying $\hlam \in [(1 - \frac \gamma {10}) \lam_1(\msig), (1 + \frac \gamma {10}) \lam_1(\msig)]$, with probability $\ge 1 - \frac \beta 2$.
\end{lemma}
\begin{proof} 
First, with probability $\ge 1 - \frac \beta 4$, we have that
\[|g| = O\Par{\frac{k\tau}{\eps}\sqrt{\log \frac 1 \delta} \cdot \sqrt{\log \frac 1 \beta}} \le \frac \gamma {20} \lam_1(\msig),\]
by plugging our lower bound on $b = \frac n m$ in Proposition~\ref{thm:priv_sparse_cov_operator_norm_estimation} into the definition of $\tau$ in \eqref{eq:tau_choice_pca}.

Next, the proof of Lemma~\ref{lem:utility_1} applies up to the point where $\bmsig$ is defined, so applying its conclusion, with probability $\ge 1 - \frac \beta 4$, we have $\norms{\bmsig - \msig}_{\infty, \infty} \le 3\tau$. Lemma~\ref{lem:kRCS_thresholding} then gives
\[\normop{\calT_{5\tau}(\bmsig) - \msig} \le 10k\tau \le \frac \gamma {20} \lam_1(\msig),\]
again by using our choice of $\tau$. The claim follows from the above displays and Weyl's inequality.
\end{proof}

\section{Deferred Proofs from Section~\ref{sssec:priv_sparse_pca_lower_bound}}\label{appendix:priv_sparse_pca_lower_bound}

In this section, we prove Proposition~\ref{prop:private_assouad}, our variant of the DP Assouad's lemma from \cite{pmlr-v132-acharya21a}. Our variant allows for subsets $\calV$ of the hypercube, whereas Theorem 3 of \cite{pmlr-v132-acharya21a} takes $\calV = \{\pm 1\}^d$.

As with Theorem 3 of \cite{pmlr-v132-acharya21a}, our variant is based on a binary hypothesis testing lower bound.

\begin{lemma}[Theorem 1, \cite{pmlr-v132-acharya21a}]\label{lem:priv_high_dim_lecam} Let $\mathcal{P}_{1}, \mathcal{P}_{2}$ be families of distributions over $\Omega^n$ for a sample space $\Omega$, and let $P_{1} \in \mathsf{co}\Pars{\mathcal{P}_{1}}$, $P_{2} \in \mathsf{co}\Pars{\mathcal{P}_{1}}$ be mixtures where $\mathsf{co}(\calP)$ is the convex hull of distributions in $\calP$. Let $\Gamma$ be a coupling of $(P_1, P_2)$ such that, denoting $\calX \defeq \{\vx_i\}_{i \in [n]}$, $\calY \defeq \{\vy_i\}_{i \in [n]}$,
\[\E_{(\calX, \calY) \sim \Gamma}\Brack{\sum_{i \in [n]} \ind(\vx_i \neq \vy_i)} \le D.\]
Then if $\alg: \Omega^n \to \{1, 2\}$ is any $(\eps, \delta)$-DP hypothesis test,
\[\max_{i \in \{1,2\}}\max_{P \in \mathcal{P}_{i}}\Prob\bb{\calA\bb{\calX} \neq i \mid \calX \sim P} \geq \frac{1}{2}\Par{ 0.9e^{-10\eps D} - 10D\delta}.\]
\end{lemma}

We can now prove Proposition~\ref{prop:private_assouad}.

\PrivateAssouad*
\begin{proof}
For an arbitrary $(\eps, \delta)$-DP estimator $\calA(\calX)$, define a postprocessing of $\calA$:
\[\hvs(\calX) \defeq \arg\min_{\vs \in \calV} \ell\Par{\calA(\calX), \vtheta_{\vs}}. \]
Observe that $\hvs$ is also $(\eps, \delta)$-DP, and that for any $\vs \in \calV$,
\[\ell\Par{\vtheta_{\hvs(\calX)}, \vtheta_{\vs}} \le 2\ell\Par{\calA(\calX), \vtheta_{\hvs(\calX)}} + 2\ell\Par{\calA(\calX), \vtheta_s} \le 4\ell\Par{\calA(\calX), \vtheta_{\vs}}.\]
    Hence, \bas{R\bb{\calV , \ell, \eps, \delta} \geq \frac{1}{4}\min_{\hvs:\Omega^n \to \calV  \text{ is } \bb{\eps, \delta}\text{-DP}}\max_{\vs \in \mathcal{V}}\E_{\calX \sim P_{\vs}^{\otimes n}}\bbb{\ell\bb{\vtheta_{\hvs(\calX)}, \vtheta_{\vs}}}.
    }
    For a fixed estimator $\hvs: \Omega^n \to \calV$, we have
    \bas{
        \max_{\vs \in \mathcal{V}}\E_{\calX \sim P_{\vs}^{\otimes n}}\bbb{\ell\bb{\vtheta_{\hvs(\calX)}, \vtheta_{\vs}}} &\geq \frac{1}{|\mathcal{V}|}\sum_{\vs \in \mathcal{V}}\E_{\calX \sim P_{\vs}^{\otimes n}}\bbb{\ell\bb{\vtheta_{\hvs(\calX)}, \vtheta_{\vs}}} \\
        &\geq \frac{2\tau}{|\mathcal{V}|}\sum_{i \in [d]}\sum_{\vr \in \mathcal{V}}\Prob\Brack{\hvs_{i}(\calX) \neq \vr_{i}\mid \vs = \vr},
    }
    where in the last line, we use $\vs$ to index the distribution generating $\calX$. Continuing,
    \bas{
    \sum_{i \in [d]}\sum_{\vr \in \mathcal{V}}\Prob\Brack{\hvs_{i}(\calX) \neq \vr_{i}\mid \vs = \vr} &= \sum_{i \in [d]}\bb{|\mathcal{V}_{+i}|\Prob_{\calX \sim P_{+i}}\Brack{\hvs_{i}(\calX) \neq 1} + |\mathcal{V}_{-i}|\Prob_{\calX \sim P_{-i}}\Brack{\hvs_i(\calX) \neq -1}}.
    }
Finally, observe that for all $i \in [d]$, $\Prob_{\calX \sim P_{+i}}\Bracks{\hvs_{i}(\calX) \neq 1} + \Prob_{\calX \sim P_{-i}}\Bracks{\hvs_i(\calX) \neq -1} $ is at least
\[ \min_{\phi: \Omega^n \to \{\pm 1\} \text{ is } (\eps,\delta)\text{-DP}} \Pr_{\calX \sim P_{+i}}[\phi(\calX) \neq 1] + \Pr_{\calX \sim P_{-i}}[\phi(\calX) \neq -1]. \]
The conclusion now follows by combining the above four displays, and applying Lemma~\ref{lem:priv_high_dim_lecam}.
\end{proof}

\section{Approximate DP Lower Bound for Standard PCA}\label{appendix:dense_pca_lower_bound}

In this section, we adapt the fingerprinting method of \cite{narayanan2024better} to prove
Theorem~\ref{thm:dense_PCA_approx_DP_lower_bound}, a lower bound for approximate-DP PCA in the
standard dense setting. The lower-bound instance below uses Gaussian samples with a spiked
inverse-Wishart prior and imposes no $\kRCS$ or sparsity structure on the covariance. Its covariance has a
constant-order eigengap on the high-probability event used in the proof. Throughout this section,
unless explicitly conditioned otherwise, expectations are over the prior, the samples, and the
internal randomness of the algorithm. The proof applies Lemma~\ref{lem:bayesian_fingerprinting} to PCA. As in
Section~\ref{sssec:priv_sparse_cov_est_lower_bound}, the proof reduces to upper and lower bounds
on a suitable score.

\begin{construction}[Spiked inverse-Wishart PCA instance]
\label{const:spiked_invwish_pca}
Fix $\gamma\in(0.1,0.5)$, $\nu = \lceil 1600d/\gamma^{2}\rceil$, and a unit vector
$\vv \in \R^{d}$. Let
\[
\mm \defeq (1-\gamma)\id_d + \gamma \vv\vv^{\top}.
\]
Given $\mm$, draw
\[
\msig \sim \InvWishart\big((\nu-d-1)\mm,\nu\big),
\qquad
\vx_1,\dots,\vx_n \mid \msig \simiid \calN(\0,\msig).
\]
Let $\vt(\calX)$ be any (possibly randomized) estimator with $\norms{\vt}_2=1$,
where $\calX=(\vx_1,\dots,\vx_n)$. Let
\[
\hmsig\defeq\frac1n\sum_{i\in[n]}\vx_i\vx_i^\top,
\qquad
\mpp\defeq\vv_1(\msig)\vv_1(\msig)^\top,
\qquad
\hmpp\defeq\vv_1(\hmsig)\vv_1(\hmsig)^\top.
\]
\end{construction}

We will use the following standard representation of the inverse-Wishart draw. There is a matrix
$\mg\in\R^{\nu\times d}$ with independent $\calN(0,1)$ entries such that, for
\[
\ma \defeq \frac1\nu \mg^\top \mg,
\qquad
r\defeq \frac{\nu-d-1}{\nu},
\]
we have
\begin{equation}\label{eq:inv_wish_construction}
\msig = r\mm^{1/2}\ma^{-1}\mm^{1/2}.
\end{equation}

We require the following helpful properties of the spiked inverse-Wishart PCA construction.

\begin{lemma}[Properties of the spiked inverse-Wishart PCA construction]
\label{lem:spiked_pca_properties}
Under Construction~\ref{const:spiked_invwish_pca}, the following hold.
\begin{enumerate}
\item\label{lem:inverse-wishart-properties:itemgoodevent}
For $\eta\defeq6\sqrt{\frac{d}{\nu}}$, let
$\calE_\eta\defeq\Braces{(1-\eta)\id_d\preceq\ma\preceq(1+\eta)\id_d}$. Then $\Prob(\calE_\eta)\geq1-2e^{-d/2},$
and, on $\calE_\eta$,
\[
\frac15\id_d\preceq\msig\preceq\frac{11}{10}\id_d,
\qquad
\lam_1(\msig)-\lam_2(\msig)\geq\frac{3}{100}.
\]
\item\label{lem:inverse-wishart-properties:qnorm}
For every fixed integer $q\geq1$, there are constants $0<g_q<G_q<\infty$ such that,
for all sufficiently large $d$,
\[
(1-\gamma)rg_q
\leq
\E[\lam_{\min}(\msig)^q]^{1/q}
\leq
\E[\lam_{\max}(\msig)^q]^{1/q}
\leq
rG_q.
\]
\item\label{lem:inverse-wishart-properties:opnorm}
\[
\E\bbb{\normop{\hmsig-\msig}^2}
\leq
O\bb{\frac dn+\bb{\frac dn}^2}.
\]
\end{enumerate}
\end{lemma}
\begin{proof}
For Item~\ref{lem:inverse-wishart-properties:itemgoodevent}, let
$u\defeq2\sqrt{d/\nu}$. By the Gaussian singular-value bound
\cite[Corollary~5.35]{vershynin2010nonasymptotic}, the singular values of
$\mg/\sqrt{\nu}$ lie in $[1-u,1+u]$ except with probability $2e^{-d/2}$.
Since $\ma=(\mg/\sqrt\nu)^\top(\mg/\sqrt\nu)$ and $2u+u^2\leq3u=\eta$,
this is exactly the event $\calE_\eta$. On this event,
$\normop{\ma^{-1}-\id_d}\leq\eta/(1-\eta)\leq\gamma/5$, where the last
inequality uses $\nu\geq1600d/\gamma^2$. Thus
$\msig=r\mm^{1/2}\ma^{-1}\mm^{1/2}$ is an $r\gamma/5$ operator-norm
perturbation of $r\mm$. The matrix $r\mm$ has top eigenvalue $r$ and
remaining eigenvalues $r(1-\gamma)$, so Weyl's inequality gives the stated
constant spectral bounds and eigengap, using $r\geq1/2$ and
$\gamma\in(0.1,0.5)$.

For Item~\ref{lem:inverse-wishart-properties:qnorm}, $\msig$ has the same
eigenvalues as $r\ma^{-1/2}\mm\ma^{-1/2}$. The deterministic spike satisfies
$(1-\gamma)\id_d\preceq\mm\preceq\id_d$, so the eigenvalues of $\msig$ are
sandwiched between those of $r(1-\gamma)\ma^{-1}$ and $r\ma^{-1}$. Since
$\ma^{-1}=\nu(\mg^\top\mg)^{-1}$, \cite[Lemma~3.7]{narayanan2024better}
supplies constant moments for the extreme eigenvalues of $\ma^{-1}$; absorbing
constants gives the claim.

For Item~\ref{lem:inverse-wishart-properties:opnorm}, condition on $\msig$ and
apply \cite[Corollary~2]{koltchinskii2017concentration}. Its right-hand side is
$O(\normop{\msig}^2\{d/n+(d/n)^2\})$, because
$\Tr(\msig)/\lambda_1(\msig)\leq d$. Taking expectations and using
Item~\ref{lem:inverse-wishart-properties:qnorm} with $q=2$ proves the bound.
\end{proof}

\begin{lemma}[Empirical projector error]
\label{lem:koltchinskii_projector_error}
Fix constants $0<\underline\lambda\leq\overline\lambda$ and $g>0$. There are constants
$c,C,C_0>0$ such that, if $\underline\lambda\id_d\preceq\msig\preceq\overline\lambda\id_d,
\lam_1(\msig)-\lam_2(\msig)\geq g,
n\geq C_0d,$
then
\[
c\frac dn
\leq
\E\Brack{\normsf{\hmpp-\mpp}^2\mid\msig}
\leq
C\frac dn.
\]
\end{lemma}

\begin{proof}
Condition on $\msig$. In the notation of \cite[Theorem~3]{koltchinskii}, the
leading eigenspace is simple and all factors depending on
$\normop{\msig}/(\lambda_1(\msig)-\lambda_2(\msig))$ are constants depending
only on the assumed spectral bounds. The leading variance coefficient is
\[
A_1(\msig)
=
2\sum_{j>1}
\frac{\lam_1(\msig)\lam_j(\msig)}
{\Pars{\lam_1(\msig)-\lam_j(\msig)}^2}.
\]
Each summand is bounded above and below by positive constants, so
$ad\leq A_1(\msig)\leq bd$ for constants $a,b>0$ depending only on the
spectral bounds and eigengap.

Let $r_{\mathrm{eff}}(\msig)=\Tr(\msig)/\normop{\msig}\leq d$. Items (1) and
(2) of \cite[Theorem~3]{koltchinskii}, applied with $m_1=1$, give a constant
$K$ such that the centered error has expansion
\[
\E\Brack{\normsf{\hmpp-\E[\hmpp\mid\msig]}^2\mid\msig}
=\frac{A_1(\msig)}{n}+\xi_n,
\qquad
\Abs{\xi_n}\leq K\bb{\bb{\frac dn}^{3/2}\vee\bb{\frac dn}^4}.
\]
The same theorem bounds the bias by
\[
\normsf{\E[\hmpp\mid\msig]-\mpp}
\leq K\bb{\frac dn\vee\bb{\frac dn}^2}.
\]
The conditional bias-variance identity then writes the desired error as the
centered term plus the squared bias. With $x=d/n$, choosing $C_0$ sufficiently
large makes $x\leq1$ and absorbs $K(x^{3/2}\vee x^4)$ into the leading
$A_1(\msig)/n$ term. The centered term is therefore between constant multiples
of $d/n$, and the squared bias is only $O((d/n)^2)$. This proves both bounds.
\end{proof}

We next record the posterior stability estimate needed by the fingerprinting argument.
\begin{lemma}[Posterior stability for the PCA prior]\label{lem:pca_posterior_conc}
Under Construction~\ref{const:spiked_invwish_pca},
\bas{
    & \E\bbb{\normsop{\E\bbb{\msig \mid \calX} - \hmsig}^{2}} = O\bb{\bb{\frac{\nu}{n}}^2\bb{1 + \bb{\frac{d}{n}} + \bb{\frac{d}{n}}^{2}}}.
}
\end{lemma}
\begin{proof}
The prior has inverse-Wishart parameters $\mPsi=(\nu-d-1)\mm$ and $\nu$. By
Fact~\ref{fact:inv_wishart_facts},
\[
\E[\msig]=\mm,\qquad
\E[\msig \mid \calX]
=
(1-w)\E[\msig]+w\hmsig,
\qquad
w \defeq \frac{n}{\nu+n-d-1}.
\]
Thus $\E[\msig\mid\calX]-\hmsig=(1-w)(\E[\msig]-\hmsig)$ and
$1-w=(\nu-d-1)/(n+\nu-d-1)\leq\nu/n$. It remains to bound
$\E\normsop{\hmsig-\E[\msig]}^2$. By the triangle inequality and
$(a+b+c)^2\leq3(a^2+b^2+c^2)$, 
\[\E\normop{\hmsig-\E[\msig]}^2 \le 
3\E\normop{\hmsig-\msig}^2
+3\E\normop{\msig}^2
+3\normop{\E[\msig]}^2 .
\]
The first term is Lemma~\ref{lem:spiked_pca_properties},
Item~\ref{lem:inverse-wishart-properties:opnorm}; the second is
Item~\ref{lem:inverse-wishart-properties:qnorm} with $q=2$; and the third is
bounded by Jensen's inequality. This gives
$O(1+d/n+(d/n)^2)$ before multiplying by $(1-w)^2$, which proves the claim.
\end{proof}

Under Construction~\ref{const:spiked_invwish_pca}, instantiate
Model~\ref{model:bayesian_fingerprinting} with $\theta=\msig$,
$P_\theta=\calN(\0,\msig)$, $g(\msig)\defeq\mpp$, and
$\psi(\msig,\vx)\defeq\vx\vx^\top-\msig$.
Let $\calM:(\R^d)^n\to\R^d$ output a unit vector $\vt$, and set $M(\calX)\defeq\tmpp\defeq\vt\vt^\top$. Let $Z_i,Z_i'$ be the corresponding scores in
Lemma~\ref{lem:bayesian_fingerprinting}.

\begin{lemma}[PCA score bounds]
\label{lem:pca_score_bounds}
Under Construction~\ref{const:spiked_invwish_pca}, there are universal constants
$0<c<1<C_0,C_1$ such that, if
$\E\normsf{\tmpp-\mpp}^4\leq\rho^4$, $\rho<c$, $n\geq C_0d$, and $d\geq C_1$, then
\[
\E\Brack{\sum_{i\in[n]} Z_i} = \Omega(d)
\]
and, for every $i\in[n]$,
\[
2\eps\E\Abs{Z_i'}
+2\sqrt{\delta}\sqrt{\E[Z_i^2]+\E[(Z_i')^2]}
=
O\bb{\rho\Pars{\eps+d\sqrt{\delta}}}.
\]
\end{lemma}
\begin{proof}
\emph{Signal.}
Let $\gap(\msig)=\lam_1(\msig)-\lam_2(\msig)$. The empirical projector
$\hmpp$ is useful because it simultaneously maximizes the empirical Rayleigh
quotient and is close to the population projector. More precisely,
$\inprods{\hmpp-\mpp}{\hmsig}\geq0$ by empirical optimality. Lemma~\ref{lem:rayleigh_gap}
turns the population Rayleigh deficit of $\hmpp$ into projector error:
since $\normsf{\hmpp-\mpp}^2=2\sin^2\angle(\vv_1(\hmsig),\vv_1(\msig))$,
it gives $\inprods{\mpp-\hmpp}{\msig}\geq
\frac{\gap(\msig)}{2}\normsf{\hmpp-\mpp}^2$. Combining the empirical and
population inequalities shows that the empirical fluctuation
$\hmsig-\msig$ has positive correlation with the empirical projector error.

It remains to replace $\hmpp$ by the algorithm's output $\tmpp$. Let $R$ be the
empirical projector error $\E\normsf{\hmpp-\mpp}^2$, and let $H$ be the
posterior-stability quantity controlled by Lemma~\ref{lem:pca_posterior_conc}.
Conditional on $\calX$ and the algorithmic randomness, $\tmpp$ is fixed and
independent of the posterior draw $\msig$. Therefore the replacement error is
$\E\inprods{\hmpp-\tmpp}{\hmsig-\E[\msig\mid\calX]}$. For each realization,
$\hmpp-\tmpp$ is a difference of rank-one projectors, so its nuclear norm is
bounded by a universal constant times its Frobenius norm. Trace-operator norm
duality and Cauchy-Schwarz bound the replacement error by
$O(\sqrt{(\rho^2+R)H})$. Here we used
$\E\normsf{\hmpp-\tmpp}^2\leq2R+2\E\normsf{\tmpp-\mpp}^2$ and
$\E\normsf{\tmpp-\mpp}^2\leq\rho^2$.
Hence
\begin{equation}\label{eq:pca_signal_reduction}
\E\inprod{\tmpp-\mpp}{\hmsig-\msig}
\geq
\E\Brack{
\frac{\lam_1(\msig)-\lam_2(\msig)}{2}\normsf{\hmpp-\mpp}^2
}
-O\Pars{\sqrt{\Pars{\rho^2+R}H}}.
\end{equation}

On the good event from Lemma~\ref{lem:spiked_pca_properties} item~\ref{lem:inverse-wishart-properties:itemgoodevent},
the covariance has a constant eigengap and bounded spectrum, so Lemma~\ref{lem:koltchinskii_projector_error}
gives an empirical projector error of order $d/n$. The complement has probability
$2e^{-d/2}$, while $\normsf{\hmpp-\mpp}^2\leq2$. Thus the main term in
\eqref{eq:pca_signal_reduction} is $\Omega(d/n)$ and $R=O(d/n+e^{-d/2})$.
Using Lemma~\ref{lem:spiked_pca_properties},
Lemma~\ref{lem:pca_posterior_conc} gives $H=O((d/n)^2)$ because $\nu=\Theta(d)$
in the construction. For $\rho$ small and $n\geq C_0d$, the replacement cost is
lower order, so $\E\inprods{\tmpp-\mpp}{\hmsig-\msig}=\Omega(d/n)$. Multiplying
by $n$ gives $\E[\sum_i Z_i]\geq\Omega(d)$.

\emph{Moments.}
For $Z_i'$, the output $M(\calX^{\sim i})$ is conditionally independent of
$\vx_i$ given $\msig$. \cite[Proposition~3.8]{narayanan2024better} gives the
conditional bound
$\E[(Z_i')^2\mid\msig,\calX^{\sim i}]\leq
2\normop{\msig}^2\normsf{M(\calX^{\sim i})-\mpp}^2$. Cauchy-Schwarz, the
fourth-moment assumption, and Lemma~\ref{lem:spiked_pca_properties},
Item~\ref{lem:inverse-wishart-properties:qnorm} with $q=4$, give
$\E[(Z_i')^2]=O(\rho^2)$, hence $\E|Z_i'|=O(\rho)$.

For $Z_i$, Cauchy-Schwarz gives
$\E[Z_i^2]\leq
\sqrt{\E\normsf{\tmpp-\mpp}^4\,
\E\normsf{\vx_i\vx_i^\top-\msig}^4}$. The Gaussian norm moment bound
\cite[Proposition~3.10]{narayanan2024better}, together with
Item~\ref{lem:inverse-wishart-properties:qnorm}, bounds the second factor by
$O(d^4)$. Thus $\E[Z_i^2]=O(\rho^2d^2)$, which is the claimed DP comparison
bound.
\end{proof}

\begin{proposition}[Moment lower bound for dense PCA]\label{prop:dense_pca_moment_lower_bound}
There are universal constants $0<c<1<C_0,C_1$ such that the following holds under
Construction~\ref{const:spiked_invwish_pca}. Let $\eps,\delta\in(0,1)$ satisfy
$\delta\leq\eps^2/d^2$, let $\rho<c$, $n\geq C_0d$, and $d\geq C_1$. Every
$(\eps,\delta)$-DP algorithm $\calM:(\R^d)^n\to\R^d$ whose unit-vector output $\vt$ satisfies
\[
\E\normsf{\vt\vt^\top-\mpp}^4\leq\rho^4
\]
must have
\[
n=\Omega\bb{\frac{d}{\rho\eps}}.
\]
\end{proposition}

\begin{proof}
The score is centered because, conditional on $\msig$, $\E[\psi(\msig,\vx_i)\mid\msig]
=
\E[\vx_i\vx_i^\top\mid\msig]-\msig
=
\0.$
Lemma~\ref{lem:pca_score_bounds} supplies the two certificates in
Lemma~\ref{lem:bayesian_fingerprinting} with $L=\Omega(d)$ and
$U=O\Pars{\rho\Pars{\eps+d\sqrt\delta}}.$
The assumption $\delta\leq\eps^2/d^2$ implies $d\sqrt\delta\leq\eps$, so
$U=O(\rho\eps)$. Lemma~\ref{lem:bayesian_fingerprinting} therefore gives
\[
n\geq\frac{L}{U}=\Omega\bb{\frac{d}{\rho\eps}}.
\]
\end{proof}

\begin{lemma}[PCA moment amplification]\label{lem:pca_moment_amplification}
Let $\rho_0\in(0,1)$ be a sufficiently small universal constant. Suppose an
$(\eps,\delta)$-DP algorithm $\calA$ uses $n$ samples and solves Problem~\ref{prob:pca}
under the Gaussian instances in Construction~\ref{const:spiked_invwish_pca} with
$\Delta=\rho_0^2/100$ and $\beta=1/3$. Then there is an $(\eps,\delta)$-DP algorithm
$\calA'$ using $O(n)$ samples whose unit-vector output $\vt$ satisfies
\[
\E\normsf{\vt\vt^\top-\mpp}^4\leq\rho_0^4.
\]
\end{lemma}

\begin{proof}
Set $B=\lceil18\log(8/\rho_0^4)\rceil$, a universal constant once $\rho_0$ is
fixed. Run $\calA$ on $B$ disjoint blocks, and let $\mpp_j$ be the returned
projector on block $j$. Output a unit vector whose projector $\widehat \mpp$
minimizes the median of
$\{\normsf{\mpp_j-\mpp_\ell}:\ell\in[B]\}$.

Since Problem~\ref{prob:pca} is a pointwise guarantee for each
Gaussian instance, each block is successful with probability at least $2/3$,
conditional on $\msig$. On a successful block,
$\normsf{\mpp_j-\mpp}^2=2\sin^2\angle(\vt_j,\vv_1(\msig))\leq\rho_0^2/50$, so
$\normsf{\mpp_j-\mpp}\leq\rho_0/5$. Hoeffding's inequality gives probability at most
$\exp(-B/18)\leq\rho_0^4/8$ that at most half the blocks are successful.

If more than half the blocks are successful, any successful $\mpp_j$ has median
distance at most $2\rho_0/5$ from the other projectors. The minimizing
$\widehat \mpp$ therefore has median distance at most $2\rho_0/5$, so it is within
$2\rho_0/5$ of at least one successful block. Hence
$\normsf{\widehat \mpp-\mpp}\leq3\rho_0/5$. Since the Frobenius distance between
rank-one projectors is always at most $\sqrt2$,
\[
\E\normsf{\widehat \mpp-\mpp}^4
\leq
\bb{\frac{3\rho_0}{5}}^4+4\cdot\frac{\rho_0^4}{8}
\leq
\rho_0^4.
\]
Parallel composition preserves $(\eps,\delta)$-DP.
\end{proof}

\begin{theorem}\label{thm:dense_PCA_approx_DP_lower_bound}
There are universal constants $\Delta_0,C_1>0$ such that the following holds. Let
$\eps\leq1$, $\delta\leq\eps^2/d^2$, and $d\geq C_1$. If an $(\eps,\delta)$-DP algorithm
$\calA:(\R^d)^n\to\R^d$ solves Problem~\ref{prob:pca} with $\Delta=\Delta_0$ and
$\beta=1/3$ under the Gaussian instances in Construction~\ref{const:spiked_invwish_pca},
then
\[
n=\Omega\bb{\frac d\eps}.
\]
\end{theorem}

\begin{proof}
Apply Lemma~\ref{lem:pca_moment_amplification} with $\Delta_0=\rho_0^2/100$ for a
sufficiently small constant $\rho_0$. This gives an $(\eps,\delta)$-DP algorithm
$\calA'$ using $N=O(n)$ samples and satisfying the hypothesis of
Proposition~\ref{prop:dense_pca_moment_lower_bound}. We may assume $n=\Omega(d)$:
the lower bound of \cite[Theorem~3]{CaiMW13}, with
ambient dimension and sparsity both equal to $d$, gives constant projector error
unless $n=\Omega(d)$ for constant signal-to-noise ratio and eigengap; the same
dense specialization is also consistent with the row-sparse $q=0$ lower bound
of \cite[Theorem~3.1]{vu2013minimax}. Thus a sufficiently small constant-error
guarantee already forces $n=\Omega(d)$.
Increasing the constant number of amplification blocks if necessary ensures
$N\geq C_0d$. Therefore Proposition~\ref{prop:dense_pca_moment_lower_bound} gives
$N=\Omega(d/\eps)$, and $N=O(n)$ completes the proof.
\end{proof}

\section{Exponential Mechanism for Sparse PCA}
\label{sec:expmech-support}

This appendix gives a simple algorithm for Problem~\ref{prob:pca} under a generalization of Model~\ref{model:general_krcs_pca}, where we only assume $\msig$ has a unique, $k$-sparse leading eigenvector, with no $\kRCS$
assumption imposed.

Our algorithm is based on the exponential mechanism, and uses a brute-force enumeration of all candidate supports (size-$k$ subsets of $[d]$). Notably, after dropping privacy terms and logarithmic factors, Theorem~\ref{thm:expmech-support-utility} has leading statistical error $\widetilde O(\sqrt{k/n})$, matching the information-theoretic non-private lower bound \cite{vu2013minimax,CaiMW13} rather than the $\widetilde O(k/\sqrt n)$ rate attained by known polynomial-time approaches \cite{BerthetR13}. This improvement is enabled by an exhaustive search over all $k$-supports via the exponential mechanism.


\begin{algorithm}[H]
\caption{\label{alg:expmech-support}
$\ExpMechPCA(\calD,k,\eps,\delta,\tau)$}
\textbf{Input:} Dataset $\calD=\{\vx_i\in\R^d\}_{i\in[n]}$, sparsity parameter
$k\in[d]$, privacy parameters $(\eps,\delta)\in(0,1)^2$, threshold $\tau > 0$
\begin{algorithmic}[1]
\State $\calS_k\gets\Brace{S\subseteq[d]:|S|=k}$
\For{$i\in[n]$}
  \State $\vz_i \gets \sign(\vx_i) \circ \min\{|\vx_i|, \tau\}$
  
\Comment{Thresholding entries: $\min\{\cdot, \tau\}$, $|\cdot|$, and $\sign(\cdot)$ all entrywise, $\circ$ is entrywise multiplication.}
\EndFor
\State $\hmsig \gets \frac 1 n \sum_{i \in [n]} \vz_i\vz_i^\top$
\For{$S\in\calS_k$}
  \State $q(\calD,S)\gets\lam_1(\hmsig_{S \times S})$
\EndFor
\State $\pi_{\calD} \gets $ distribution over $\calS_k$ with $\pi_{\calD}(S) \propto \exp(\frac{\eps n}{8k\tau^2} q(\calD, S))$
\State $\hS \sim \pi_{\calD}$
\State $\sigma_{\mathrm{priv}}\gets
\frac{8k\tau^2}{\eps n}\sqrt{\log(\frac 4 \delta)}$.
\State $\mg \gets d \times d$ matrix with i.i.d.\ entries $\sim \Nor(0, \sigma_{\mathrm{priv}}^2)$
\State $\tmsig \gets \hmsig + \half(\mg + \mg^\top)$
\State \Return $\hvv \defeq d$-dimensional vector with $\vv_1(\tmsig_{\hS \times \hS})$ in the $\hS$ coordinates and $0$ elsewhere
\end{algorithmic}
\end{algorithm}

Algorithm~\ref{alg:expmech-support} is an instance of the \emph{exponential mechanism}, whose guarantees we now recall.

\begin{lemma}[Exponential mechanism, Theorem 3.10 and Corollary 3.12, \cite{dr14}]
\label{lem:expmech}
Let $\calR$ be a finite set and $q : \Gamma^n \times \calR \to \R$ be a score function with
global sensitivity \[\Delta_q
\defeq
\sup_{r \in \calR} \; \sup_{\textup{neighboring }\calD,\calD' \in \Gamma^n} \Abss{q(\calD,r)-q(\calD',r)}.\]

Then, the mechanism $\calM$ defined by $\Prb\Pars{\calM(\calD)=r}
\propto\exp\Pars{\frac{\eps}{2\Delta_q}q(\calD,r)}$
is $\eps$-DP. Moreover, for every $\beta \in (0,1)$, with probability at least
$1-\beta$,
\[
q\Par{\calD,\calM(\calD)}
\ge
\max_{r \in \calR} q(\calD,r)
-
\frac{2\Delta_q}{\eps}
\Par{
\log \Abs{\calR} + \log \frac{1}{\beta}
}.
\]
\end{lemma}
More concretely, Algorithm~\ref{alg:expmech-support} assigns to each $S \in \calS_k$ a score $q(\calD, S)$ given by the largest eigenvalue of $\hmsig^\tau_S$, the $S \times S$ submatrix of the clipped empirical covariance. It then applies the exponential mechanism (Lemma~\ref{lem:expmech}) to privately select a set $\hS$, and the remaining steps proceed similarly to Algorithm~\ref{alg:friendly_DP_projector}. We begin by proving the privacy of Algorithm~\ref{alg:expmech-support}.

\begin{lemma}[Privacy]
\label{lem:expmech-support-privacy}
Algorithm~\ref{alg:expmech-support} is $(\eps, \delta)$-DP.
\end{lemma}

\begin{proof}
Fix $S \in \calS_k$, and let $\calD,\calD'$ be adjacent datasets. Then
\[
\hmsig_S^\tau(\calD) - \hmsig_S^\tau(\calD')
=
\frac{1}{n}\Par{\vz \vz^\top - \vz'(\vz')^\top},
\]
where $\vz,\vz' \in \R^k$ are the clipped restrictions of the differing samples to $S$. Since
$\norms{\vz}_2^2,\norms{\vz'}_2^2 \le k\tau^2$,
\[
\normsop{\hmsig_S^\tau(\calD) - \hmsig_S^\tau(\calD')}
\le
\frac{\norm{\vz}_2^2 + \norm{\vz'}_2^2}{n}
\le
\frac{2k\tau^2}{n}.
\]
By Weyl's inequality,
\[
\Abs{q(\calD,S)-q(\calD', S)}
\le
\normsop{\hmsig_S^\tau(\calD)-\hmsig_S^\tau(\calD')}
\le
\frac{2k\tau^2}{n}.
\]
This proves that the statistic $q(\calD, S)$ has global sensitivity $\frac{2k\tau^2}{n}$ for every $S$.
Thus, the support selection step is $\frac \eps 2$-DP by Lemma~\ref{lem:expmech}. Also, the same calculations as above with the Frobenius norm in place of the operator norm gives the same bound of $\frac{2k\tau^2}{n}$ on the Frobenius norm sensitivity of the map
$\calD\to\hmsig_S^\tau(\calD)$. For every fixed transcript $\hS=S$, releasing
$\hmsig_S^\tau+\mg_S$ with independent entrywise Gaussian noise of variance
$\sigma_{\mathrm{priv}}^2$ is then $(\frac \eps 2,\delta)$-DP by Fact~\ref{fact:gaussian_mech}. Symmetrization and eigenvector computation are postprocessings, and basic composition gives
$(\eps,\delta)$-DP.
\end{proof}

We now conclude our analysis of Algorithm~\ref{alg:expmech-support} by proving a utility bound.

\begin{theorem}\label{thm:expmech-support-utility}
Let $(\eps,\delta,\Delta,\beta)\in(0,1)^4$, $k\in[d]$, and $\gamma\in(0,1)$.
Let $\vx_1,\ldots,\vx_n$ be drawn i.i.d.\ from a mean-zero $\sigma$-sub-Gaussian
distribution with covariance $\msig\in\PSD^{d\times d}$. Assume
$\lam_1(\msig)>0$, $\lam_2(\msig)\leq(1-\gamma)\lam_1(\msig)$, and
$\vv_1(\msig)$ is $k$-sparse. Algorithm~\ref{alg:expmech-support}, with
$\tau \gets 2\sigma \sqrt{\log \frac {4nd} \beta}$, solves Problem~\ref{prob:pca} with
\[n = \Omega\Par{\frac{\sigma^4}{\lam_1(\msig)^2} \cdot \frac{k\log(\frac d \beta)}{\gamma^2\Delta^2} + \frac{\sigma^2}{\lam_1(\msig)} \cdot \frac{k^2\log(\frac d {\beta\delta}) \log(\frac{d\sigma^2}{\eps\gamma\Delta\beta\delta\lam_1(\msig)})}{\eps\gamma\Delta}},\]
for a sufficiently large constant.
In shorthand, ignoring logarithmic factors, the leading rate is the scale-to-gap factor times $\sqrt{k/n}+k^2/(\eps n)$; the logs are in $d,k,1/\delta$, and $1/\beta$.
\end{theorem}
\begin{proof}
The privacy claim follows from Lemma~\ref{lem:expmech-support-privacy}.

Let $\vv \defeq \vv_1(\msig)$ and let $S^\star \in \calS_k$ contain
$\supp(\vv)$. We condition on four events, each with failure probability at most
$\beta/4$. First, clipping is inactive for every sample; this follows from the
coordinatewise sub-Gaussian tail bound and the choice of $\tau$. Second, the
empirical covariance is accurate on every candidate support:
\begin{equation}\label{eq:sampbound}
\normop{\hmsig_{S \times S} - \msig_{S \times S}}
\le \frac{\gamma \Delta \lam_1(\msig)}{5}
\quad\text{for all }S\in\calS_k.
\end{equation}
This is Lemma~\ref{lem:opnorm_err_sparse} with the stated sample size.
Third, the Gaussian perturbation on the selected support has operator norm at most
$\gamma\Delta\lambda_1(\msig)/5$:
\begin{equation}\label{eq:noisebound}
\normop{\Brack{\half\Par{\mg + \mg^\top}}_{\hS \times \hS}}
\le \frac{\gamma \Delta \lam_1(\msig)}{5}.
\end{equation}
This follows from the standard Gaussian operator-norm bound
\cite[Theorem~4.4.3]{Vershynin18} and the choice of $n$. Fourth, the exponential
mechanism loses at most $\gamma\Delta\lambda_1(\msig)/5$ in score; here
$\log|\calS_k|\leq k\log(ed/k)$ and the score sensitivity is $2k\tau^2/n$.

These events imply that the selected support has nearly optimal population
Rayleigh value. Indeed, $S^\star$ contains the true support, so
$\lambda_1(\msig_{S^\star\times S^\star})=\lambda_1(\msig)$. The sample event
shows that $S^\star$ has empirical score at least
$\lambda_1(\msig)-\gamma\Delta\lambda_1(\msig)/5$; the exponential mechanism
loses another $\gamma\Delta\lambda_1(\msig)/5$; and the Gaussian perturbation
loses another $\gamma\Delta\lambda_1(\msig)/5$. Thus
$\lambda_1(\tmsig_{\hS\times\hS})\geq
\lambda_1(\msig)-3\gamma\Delta\lambda_1(\msig)/5$.

The output $\hvv$ is the top eigenvector of $\tmsig_{\hS\times\hS}$, embedded
into $\R^d$. Since $\hvv$ is supported on $\hS$, the sample and noise events
convert the preceding lower bound into the population Rayleigh bound
$\hvv^\top\msig\hvv\geq\lambda_1(\msig)-\gamma\Delta\lambda_1(\msig)$.
Now Lemma~\ref{lem:rayleigh_gap} gives the conclusion directly: it says that
the squared sine error is at most the Rayleigh deficit divided by the eigengap.
Let $\hvv=c\vv+\mathbf r$ where $\mathbf r$ is orthogonal to $\vv$, $\|\mathbf r\|^2=\sin^2\angle(\hvv,\vv)$, and $c^2=1-\|\mathbf r\|^2$. Then the population Rayleigh lower bound and the eigengap imply $\hvv^\top\msig\hvv\leq\lambda_1(\msig)-(\lambda_1(\msig)-\lambda_2(\msig))\sin^2\angle(\hvv,\vv)$.
Here the deficit is at most $\gamma\Delta\lambda_1(\msig)$, while the eigengap is
at least $\gamma\lambda_1(\msig)$. Therefore
$\sin^2\angle(\hvv,\vv)\leq\Delta$.
\end{proof}

\end{document}